\pdfoutput=1
\documentclass[10pt]{article}

\usepackage{"gene"}

\let\hat\widehat

\definecolor{cm}{RGB}{0,0,200}
\definecolor{purple}{RGB}{200,0,200}

\newcommand{\stateS}{\ensuremath{\calS}}
\newcommand{\actionS}{\ensuremath{\calA}}
\newcommand{\numS}{\ensuremath{S}}
\newcommand{\numA}{\ensuremath{A}}
\newcommand{\state}{\ensuremath{s}}
\newcommand{\action}{\ensuremath{a}}
\newcommand{\actioni}[1]{\ensuremath{a}^{(#1)}}
\newcommand{\reward}{\ensuremath{r}}

\newcommand{\rewardDist}{\ensuremath{R}}
\newcommand{\empReward}{\ensuremath{\widehat \reward}}
\newcommand{\feature}{\phi}
\newcommand{\optparam}{{\theta^\star}}

\newcommand{\empCov}{\Sigma_\calD}

\newcommand{\dataset}{\ensuremath{\calD}}
\newcommand{\behDist}{\ensuremath{\mu}}
\newcommand{\testDist}{\ensuremath{\rho}}

\newcommand{\policy}{\ensuremath{\pi}}
\newcommand{\optPolicy}{\ensuremath{\policy^\star}}
\newcommand{\estPolicy}{\ensuremath{\widehat \policy}}
\newcommand{\policyclass}{\ensuremath{\Pi}}
\newcommand{\val}{\ensuremath{V}}
\newcommand{\optVal}{\ensuremath{\val^\star}}
\newcommand{\estVal}{\ensuremath{\widehat \val}}

\newcommand{\datamatrix}{\Phi}
\newcommand{\errorvec}{\eta}
\newcommand{\rewardvec}{r}
\newcommand{\thetaOLS}{\widehat{\theta}_\mathsf{ols}}

\newcommand{\plugin}{\ensuremath{\widehat \policy_{\mathsf{plug}}}}

\newcommand{\n}{\ensuremath{n}}

\newcommand{\confLevel}{\ensuremath{\delta}}

\newcommand{\lambdamin}{\lambda_\mathrm{min}}

\newcommand{\thetalp}{\Theta_{p}}

\newcommand{\ber}{\mathsf{Ber}}
\newcommand{\behDistmin}{\behDist_{\min}}

\newcommand{\CBinstance}{\ensuremath{ \calQ }}

\newcommand{\const}{\ensuremath{c}}

\newcommand{\BCP}{BCP}
\newcommand{\LCB}{LCB}
\newcommand{\PEVI}{PEVI}
\newcommand{\name}{PUNC}

\newcommand{\complexity}{{\mathfrak{C}}}

\newtoggle{neurips}
\settoggle{neurips}{false}
\newcommand{\version}[2]{\iftoggle{neurips}{#1}{#2}}
\begin{document}

\title{Pessimism for Offline Linear Contextual Bandits \\using $\ell_p$ Confidence Sets}

\author{Gene Li\thanks{Toyota Technological Institute at Chicago, Chicago, IL 60637, USA; Email: \texttt{gene@ttic.edu}} 
\and Cong Ma\thanks{Department of Statistics, University of Chicago, Chicago, IL 60637, USA; Email: \texttt{congm@uchicago.edu}}
\and Nathan Srebro\thanks{Toyota Technological Institute at Chicago, Chicago, IL 60637, USA; Email: \texttt{nati@ttic.edu}}}
\date{\today}

\maketitle

\begin{abstract}
	We present a family $\{\estPolicy_p\}_{p\ge 1}$ of pessimistic learning rules
	for offline learning of linear contextual bandits, relying on
	confidence sets with respect to different $\ell_p$ norms, where $\estPolicy_2$ corresponds to
	Bellman-consistent pessimism (BCP), while $\estPolicy_\infty$ is a
	novel generalization of lower confidence bound (LCB) to the linear setting.  We show that the novel
	$\estPolicy_\infty$ learning rule is, in a sense, adaptively optimal, as it
	achieves the minimax performance (up to log factors) against all
	$\ell_q$-constrained problems, and as such it strictly dominates
	all other predictors in the family, including $\estPolicy_2$.
\end{abstract}

\section{Introduction}
Offline (or batch) reinforcement learning
(RL)~\citep{lange2012batch,levine2020offline} seeks to learn a good
policy from fixed historical data without active interactions with the
environment. This offline paradigm has been widely adopted in
applications including dialog generation~\citep{jaques2019way},
autonomous driving~\citep{yurtsever2020survey}, and robotic
control~\citep{kumar2020conservative}, etc.

When the offline dataset has insufficient coverage over the state and action spaces, planning via nominal estimates of either the value function or the model may perform poorly---a phenomenon that is observed even in a simple two-armed bandit~\citep{rashidinejad2021bridging}. This challenge motivates the adoption of the \emph{pessimism principle} for solving offline RL. In essence, the pessimism principle discounts policies that are less represented/supported in the offline dataset, and hence is pessimistic/conservative in outputting a policy.  Built on this common principle, a diverse collection of pessimistic learning rules have been proposed in theory and practice~\citep{jin2021pessimism,rashidinejad2021bridging,xiao2021optimality,xie2021bellman,zanette2021provable,zhan2022offline,fujimoto2019off,kumar2019stabilizing,wu2019behavior, kidambi2020morel,liu2020provably,yu2020mopo}. 
This leads us to the following natural question:  
\begin{center}
  \version{\vspace{-0.5em}}{} 
  {\em Which pessimistic learning rule should one use for solving offline RL problems?}
  \version{\vspace{-0.5em}}{}
\end{center}
\noindent In this paper, we address the question in the setting of
offline linear contextual bandits, in which the expected reward---as a
function of the state-action pair---is linear with respect to a known
feature mapping that maps state-action pairs to finite-dimensional
vectors. Our goal is to make sense of previously proposed learning rules
for offline RL, and understand which learning rule is ``optimal'' in a
statistical sense. We present a general family $\{\estPolicy_p\}_{p\ge 1}$ of
pessimistic learning rules based on the construction of $\ell_p$
confidence sets for the unknown linear parameter. We advocate for $\estPolicy_\infty$, a new $\ell_\infty$ learning rule for offline linear
contextual bandits, which we call Pessimism via Uniform Norm Confidence (for short, \name{}).\footnote{Throughout the paper, we use $\estPolicy_\infty$ and \name{} interchangeably.} \name{} directly extends the lower confidence bound
algorithm proposed in the tabular contextual bandit setting
\cite{rashidinejad2021bridging}.  We show that \name{} (1) achieves a suboptimality guarantee that
dominates other $\hat{\pi}_p$ (up to log factors, which we ignore
throughout the introduction), and (2) has an \emph{adaptive minimax
  optimality} property that is unique among the family $\{\hat{\pi}_p\}_{p\ge 1}$.
In particular, we argue that \name{} dominates prior
learning rules which are based on $\ell_2$ pessimism (e.g.,
\citep{xie2021bellman, zanette2021provable,jin2021pessimism}) and
which cannot attain adaptive minimax optimality.

\version{\paragraph{Roadmap.}}{\subsection*{Roadmap}}
We first introduce a broad class of pessimistic learning rules in~\Cref{sec:pess-est}. 
The construction of these pessimistic learning rules relies on the observation that any 
\emph{confidence set} of the linear reward function automatically induces a pessimistic 
value estimate, and hence a pessimistic learning rule. 
As concrete examples, for each $p \geq 1$, one can design $\estPolicy_p$, an $\ell_p$ learning rule, by 
constructing such a confidence set using the $\ell_p$ distance metric. We show in \Cref{sec:ub-discussion} 
that $\estPolicy_2$ recovers the Bellman-consistent pessimism (\BCP{}) learning rule~\citep{xie2021bellman}, 
proposed for offline RL with general function approximation; meanwhile, $\estPolicy_\infty$ generalizes the lower confidence bound (\LCB{}) learning rule, proposed for offline tabular RL, to the linear setting.

Once we have cast pessimistic estimation in this framework, we can study the performance 
guarantees of the family $\{\estPolicy_p\}_{p\ge 1}$. Employing a notion of \emph{pessimism-validity} (\Cref{def:pess-valid}) allows us to easily to derive upper bounds on suboptimality for each $\estPolicy_p$ in terms of the dual $\ell_q$ norm (where $1/p + 1/q = 1$); see~\Cref{thm:lp-estimator-bound}. For $p=2$, the upper bound improves over that provided in the paper~\citep{xie2021bellman} for linear contextual bandits. For $p=\infty$, the upper bound matches that proved in the paper~\citep{rashidinejad2021bridging} for tabular contextual bandits. A key observation regarding the upper bound is that the suboptimality guarantee of $\estPolicy_\infty$ \emph{dominates} all other $\estPolicy_p$ in the general linear setting. This partially showcases the advantage of using \name{}.

To further investigate the advantage of \name{} over other $\estPolicy_p$ (for $p \in [1,\infty)$), we consider the fundamental statistical limits of the offline linear contextual bandit problem in \Cref{sec:lower-bounds}. Inspired by both the upper bounds we prove and prior work~\citep{zanette2021provable,rashidinejad2021bridging,yin2021towards}, we consider a sequence of norm-constrained classes of contextual bandit instances indexed by the $\ell_q$ norm ($q \geq 1$).  We prove that each $\estPolicy_p$ is minimax rate-optimal within the dual $\ell_q$-norm constrained contextual bandit class; see \Cref{thm:lp-lower-bound}. 
However, \Cref{thm:lp-lower-bound} delivers an even stronger message: \name{} is \emph{adaptively minimax optimal} in the sense that it simultaneously achieves optimality for all $\ell_q$-norm constrained classes, as illustrated by \Cref{fig:adaptive-opt}.  We also demonstrate that such adaptivity is unique to \name{} as other values of $p$ (e.g., $p=2$) cannot achieve simultaneous optimality. Instead, $\estPolicy_p$ is only adaptively optimal for $\ell_q$-norm constrained classes where $q \ge p/(p-1)$; see~\Cref{thm:separation-l2-linfty-new}.

\version{}{\medskip}
\noindent In summary, our main contributions are the following:
\version{\vspace{-\topsep}}{}
\version{\begin{itemize}[leftmargin=0.5cm]}{\begin{itemize}}
  \item We introduce a novel learning rule, \name{}, for solving the
    offline linear contextual bandit problem, whose performance
    guarantee dominates those of all other $\estPolicy_p$, for finite $p$ (\Cref{thm:lp-estimator-bound}).
  \item We show minimax lower bounds over norm-constrained classes of
    contextual bandit instances, which show that each $\estPolicy_p$ is optimal over the dual $\ell_q$ class, up to log
    factors in the dimension (\Cref{thm:lp-lower-bound}).
  \item We demonstrate that \name{} satisfies the
    adaptive minimax optimality property
    (\Cref{sec:adaptive-minimax-opt}), and show that this property is
    unique to \name{} by proving a separation
    result against any other $\estPolicy_p$
    (\Cref{thm:separation-l2-linfty-new}, and see also \Cref{fig:adaptive-opt}).
\end{itemize}

\section{Problem setup}\label{sec:problem-setting}
We begin by introducing the problem of offline learning in linear contextual bandits. 
Let $\stateS$ and $\actionS$ be the state space and the action space, respectively. 
Let $\feature: \stateS \times \actionS \to \R^d$ be a known feature mapping. 
In the offline setting, we observe a dataset $\dataset \coloneqq \inbraces{ (\state_i, \action_i, \reward_i) }_{ i = 1}^n$, where the covariates $\{(\state_i, \action_i)\}_{i=1}^n$ are fixed and the rewards are drawn independently according to $\reward_i \sim R(\state_i, \action_i)$, where $R(\state, \action)$ is the reward distribution associated with the pair $(\state, \action)$. We assume that $R(\state, \action)$ is 1-subgaussian for every $(\state, \action)$ with mean reward $\reward(\state,\action) \coloneqq \E[R(\state, \action)]$. Furthermore, we assume that the expected reward is linear in the sense that for every $(\state, \action)$ pair, $\reward(\state,\action) = \feature(\state, \action)^\top \theta^\star$ for some unknown parameter vector $\theta^\star \in \R^d$.

Let $\policy: \stateS \to \actionS$ be a deterministic policy. 
Fixing a (known) {test distribution} $\testDist\in\Delta(\stateS)$, we define the 
{value} of the policy $\policy$ (with respect to $\testDist$) as 
\begin{align}\label{eq:value}
  \val ( \policy ) \coloneqq \E_{\state \sim \testDist } \insquare{ 
  \reward ( \state, \policy( \state ) ) } = \E_{\state \sim \testDist } \insquare{ 
    \feature( \state, \policy(\state) )^\top  \optparam }.
\end{align}
Correspondingly, we define the optimal policy $\optPolicy$ as 
\begin{align}\label{eq:opt-policy}
  \optPolicy( \state ) \coloneqq \arg\max_{ \action \in \actionS } \; \reward (\state, \action) 
  = \arg\max_{ \action \in \actionS } \; \feature(\state, \action)^\top \optparam, \quad \text{for each }s \in \stateS.
\end{align}
The goal of offline learning in linear contextual bandits is to design a learning rule which takes as input a dataset $\dataset$ and outputs a policy
$\estPolicy$ that maximizes the value~\eqref{eq:value}; in this paper we abuse notation and also denote the learning rule as $\estPolicy$. We measure the {suboptimality} of $\estPolicy$ 
using $\val( \optPolicy ) - \val( \estPolicy )$.



\version{}{
\paragraph{Notation.}
We define the data matrix $\datamatrix \in \mathbb{R}^{n \times d}$, where the $i$-th row of $\datamatrix$ is given by $\feature(\state_i, \action_i)^{\top}$. 
We also define the observed reward vector $\rewardvec \coloneqq (\reward_{1}, \ldots, \reward_{n})^{\top} \in \mathbb{R}^{n}$.

For a positive semidefinite matrix $A$, we denote $A^{1/2}$ to be the principal square root, \ie $A^{1/2}$ is the the unique positive semidefinite matrix $B$ such that $A = B^2$. For a norm $\norm{\cdot}$ over $\R^d$, we use $\norm{\cdot}_{*}$ to denote its dual norm. We say a random variable $X$ is $\sigma^2$-subgaussian if it satisfies $\P[\abs{X - \E X} \ge t] \le 2e^{-t^2/2\sigma^2}$ for all $t \geq 0$. For a space $\calX$, we use $\Delta(\calX)$ to denote the set of probability distributions over $\calX$. We denote $\mathbb{S}^{d-1}$ to be the unit sphere and $\calB_2^d$ to denote the unit ball under the $\ell_2$ norm. We use the convention that $\tfrac{1}{\infty}= 0$ and $\tfrac{1}{0} = \infty$.}

\section{Offline learning with pessimism}\label{sec:pess-est}

The pessimism principle has recently gained much attention in offline RL theory and practice.\version{}{\footnote{An alternative to pessimistic learning rules is the usual plug-in rule. For completeness, in \Cref{sec:plugin}, we include an analysis of the plug-in rule as well as an example which highlights when pessimism is beneficial.}} At a high level, pessimistic learning rules first construct a data-dependent estimate $\estVal(\policy)$ of the true value function $\val(\policy)$ that is pessimistic, \ie $\estVal(\policy) \le \val(\policy)$ for all $\policy$. Then, the learning rule proceeds to select the policy that maximizes this pessimistic value function, \ie
\begin{align}\label{eq:selection-rule}
  \estPolicy \coloneqq \argmax_{\policy \in \policyclass} \ \estVal(\policy).
\end{align}
Here, $\policyclass \subseteq \actionS^\stateS$ is some policy class that contains the optimal policy $\optPolicy$. To see why this choice of policy makes sense, let us decompose the suboptimality of $\estPolicy$ as follows: 
\begin{equation}\label{eq:subopt-decomposition}
\val(\optPolicy) - \val(\estPolicy) = \big( \val(\optPolicy) - \estVal(\optPolicy) \big) + \big( \estVal(\optPolicy) - \estVal(\estPolicy) \big) + \big(\estVal(\estPolicy) - \val(\estPolicy) \big).
\end{equation}
The middle term is non-positive by definition of $\estPolicy$. 
Due to the pessimistic property of $\estVal$, we also have $\estVal(\estPolicy) - \val(\estPolicy) \leq 0$, which yields the suboptimality upper bound
\begin{align}\label{eq:general-upper-bound}
\val(\optPolicy) - \val(\estPolicy) \leq \val(\optPolicy) - \estVal(\optPolicy) .
\end{align}
Consequently, under the selection rule~\eqref{eq:selection-rule}, a tight pessimistic value function $\estVal$ induces a policy 
with small suboptimality.

\subsection{Achieving pessimism by building confidence sets}
As a general strategy, one can construct the pessimistic value estimator $\estVal$ by building confidence sets for the linear parameter~$\theta^\star$. 
Let $\Theta \subseteq \R^d$ be a \emph{confidence set} that contains the true parameter~$\theta^\star$. We can define the corresponding pessimistic value estimator
\begin{align}\label{eq:pess-value}
  \estVal(\policy) \coloneqq \inf_{\theta \in \Theta } \ \E_{\state \sim \testDist} [ \feature(\state, \policy(\state) )^{\top} \theta ],
\end{align}
and its associated policy learning rule \version{$\estPolicy_\Theta \coloneqq \argmax_{ \policy \in \policyclass } \; \estVal(\policy)$.}{\begin{align}\label{eq:global-pess}
\estPolicy_\Theta \coloneqq \argmax_{ \policy \in \policyclass } \; \estVal(\policy).
\end{align}}
Here for simplicity we take $\Pi$ to be the class of all deterministic policies. 

In essence, the confidence set $\Theta$ captures the amount of uncertainty we have about the ground truth $\theta^\star$. Once $\Theta$ is determined, we construct the value estimate $\estVal(\policy)$ via the worst-case value of $\policy$ among all plausible linear parameters $\theta$ in the confidence set $\Theta$. It is immediate to see that under the assumption $\theta^\star \in \Theta$, one has $\estVal(\policy) \leq \val(\policy)$ for all $\pi$. 
In other words, the value estimator $\estVal$ is indeed pessimistic.  
As a result, we can apply the general upper bound~\eqref{eq:general-upper-bound} to obtain 
\begin{align}\label{eq:linear-upper-bound}
\val(\optPolicy) - \val(\estPolicy_\Theta) \leq \val(\optPolicy) - \estVal(\optPolicy) = \sup_{\theta \in \Theta } \; \E_{\state \sim \testDist} [ \feature(\state, \optPolicy(\state) )^{\top} ( \theta^\star - \theta)  ],
\end{align}
where the identity follows from the definition~\eqref{eq:pess-value}.
Clearly, the ``smaller'' the confidence set, the smaller the bound on suboptimality. An extreme case is when $\Theta$ contains the singleton $\theta^\star$, which yields zero suboptimality. However, since only noisy rewards are observed, we cannot hope to construct such a good confidence set. Given the uncertainty about the rewards, our confidence set has to be ``large'' enough in order to guarantee that $\theta^\star \in \Theta$ with decent probability.


Below we present a general definition called pessimism-validity that involves both the size of the confidence set and also its confidence level, both of which allow us to bound the suboptimality of the pessimistic learning rule $\estPolicy_\Theta$. Let $\norm{\cdot}$ be any norm over $\R^d$ that will be used to measure the size of the confidence set $\Theta$. Let $\confLevel \in (0,1)$ be the failure probability. 
We have the following definition.


\begin{definition}\label{def:pess-valid}
We say the confidence set $\Theta$ is $(\beta, \confLevel)$ pessimism-valid under the norm $\norm{\cdot}$ if with probability at least $1 - \confLevel$, the following two properties hold: (1) $\theta^\star \in \Theta$; (2) $\sup_{\theta \in \Theta} \norm{\theta^\star - \theta} \le \beta$.
\end{definition}

\noindent A $(\beta, \confLevel)$ pessimism-valid confidence set $\Theta$ automatically induces a pessimistic learning rule $\estPolicy_\Theta$ with bounded suboptimality, as shown in the following proposition. 

\begin{proposition}\label{prop:general-global-pessimism}
Suppose that $\Theta$ is $(\beta, \confLevel)$ pessimism-valid under the norm $\norm{\cdot}$. Then with probability at least $1-\confLevel$, the pessimistic learning rule $\hat{\pi}_\Theta$ obeys
\begin{align*}
  \val(\optPolicy) - \val(\hat{\pi}_\Theta) \le \beta \cdot \Big\lVert \Ex{s\sim \rho}{\phi(s, \optPolicy(s))} \Big\rVert_{*}, 
\end{align*}
where $\norm{\cdot}_{*}$ is the dual norm of $\norm{\cdot}$.
\end{proposition}

\noindent \Cref{prop:general-global-pessimism} simply follows from the upper bound~\eqref{eq:linear-upper-bound}, the definition of pessimism-validity, and the definition of the dual norm.

\subsection{Building $\ell_p$ confidence sets}\label{subsec:confsets}

In this section, we instantiate the general strategy introduced above for achieving pessimism by constructing an $\ell_p$ confidence set around the true parameter $\theta^\star$ for some $p \geq 1$. Such constructions using $\ell_p$ norms include the aforementioned \BCP{} and \LCB{} learning rules (as well as other recently proposed learning rules) as special cases. 
As we will see, setting up the notion of pessimism-validity allows us to easily bound the suboptimality of the induced policy learning rules. 

\version{Let us denote the data matrix $\datamatrix \in \mathbb{R}^{n \times d}$, where the $i$-th row of $\datamatrix$ is given by $\feature(\state_i, \action_i)^{\top}$. We also define the observed reward vector $\rewardvec \coloneqq (\reward_{1}, \ldots, \reward_{n})^{\top} \in \mathbb{R}^{n}$.}{} Let $\thetaOLS \coloneqq (\datamatrix^\top \datamatrix)^{-1}\datamatrix^\top \rewardvec$ be the ordinary least-squares estimate for the true parameter $\theta^\star$. Throughout the paper, we assume that the sample ``covariance'' matrix $\empCov \coloneqq \frac{1}{n}\sum_{i=1}^n \feature(\state_i,\action_i)\feature(\state_i,\action_i)^\top = \frac{1}{n} \datamatrix^\top \datamatrix$ is invertible. (The results in the paper can be modified to accomodate the scenario when $\empCov$ is not invertible by considering regularized quantities $\empCov + \lambda I$ for some $\lambda > 0$.) We then consider the confidence sets of the form:
\begin{align}\label{eq:thetalp}
  \thetalp \coloneqq \inbraces{ \theta \in \R^{d} \mid \norm{\empCov^{1/2}(\theta - \thetaOLS)}_p \le {\beta} / {2} },
\end{align}
where $\beta > 0$ is a width parameter. In other words, the set $\thetalp$ contains all the $\theta$'s that are close to the OLS estimate $\thetaOLS$ in $\ell_p$ distance after the linear transformation $\empCov^{1/2}$. Since $\thetaOLS$ is a faithful approximation of the truth $\theta^\star$, we expect that $\theta^\star$ lies in this confidence set $\thetalp$ with an appropriate choice of $\beta$. This is indeed true, as the following lemma shows.

\begin{lemma}\label{lemma:valid-lp} Fix any $\delta \in (0,1)$. Set the width parameter 
$
  \beta = d^{1/p} \sqrt{ \tfrac{ 8 \log (d / \confLevel) }{n} }.
$
Then the confidence set $\thetalp$ is $(\beta, \delta)$ pessimism-valid with respect to the norm $\norm{v} \coloneqq \| \empCov^{1/2} v\|_{p}$.
\end{lemma}

\version{See \Cref{sec:pf-lp-ub} for the proof of this lemma.}{
\noindent See \Cref{sec:pf-lp-ub} for the proof of this lemma.\medskip}

\noindent Combining \Cref{lemma:valid-lp} and \Cref{prop:general-global-pessimism} immediately yields the following performance guarantee for the pessimistic learning rule constructed using $\thetalp$ (which for notational brevity we denote as $\estPolicy_p$).
\begin{theorem}\label{thm:lp-estimator-bound}
For any $p \geq 1$, with probability at least $1-\confLevel$, we have
\version{
  \begin{align*}
    \val(\optPolicy) - \val(\estPolicy_p) \le  d^{1/p} \sqrt{ \frac{ 8  \log (d / \confLevel) }{n} } \cdot \norm{\empCov^{-1/2} \E_{\state \sim \testDist} \insquare{ \feature( \state, \optPolicy(\state) )}}_q,
  \end{align*}
}{
  \begin{align*}
  \val(\optPolicy) - \val(\estPolicy_p) \le  d^{1/p} \sqrt{ \frac{ 8  \log (d / \confLevel) }{n} } \cdot \norm{\empCov^{-1/2} \E_{\state \sim \testDist} \insquare{ \feature( \state, \optPolicy(\state) )}}_q,
\end{align*}
} 
where $q$ is the solution to $1/p + 1/q = 1$.
\end{theorem}

Several remarks regarding \Cref{thm:lp-estimator-bound} are in order. 
First, the performance upper bound has a natural scaling w.r.t.~the sample size $n$, \ie $\val(\optPolicy) - \val(\estPolicy_p) \lesssim \sqrt{1 / n}$. In addition, \Cref{thm:lp-estimator-bound} provides a family of upper bounds for each specific choice of $p \geq 1$. Lastly, from an upper bound perspective, the $\estPolicy_\infty$ learning rule (which we call \name{}) dominates all the other $p \in [1, \infty)$, since for any $v\in \R^d$ and $q \in [1,\infty)$, the inequality $\norm{v}_1 \le d^{1 - 1/q} \norm{v}_q$ holds. This partially showcases the benefits of using \name{} over the alternatives. Later in \Cref{sec:lower-bounds}, we will see a stronger motivation---from the perspective of the lower bound---for using \name{}, in which we show that \name{} is adaptively minimax optimal. We also remark that the max-min form for $\estPolicy_p$ has an equivalent max-only formulation, which will be helpful for our proofs and comparisons to other algorithms: 
\begin{align}\label{eq:thetalp-estimator}
  \estPolicy_p = \argmax_{\policy \in \policyclass} \inbraces{ \E_{\state \sim \testDist}\insquare{\feature(\state,\policy(\state))}^\top \thetaOLS - \tfrac{\beta}{2} \cdot \norm{\empCov^{-1/2} \E_{\state \sim \testDist}\insquare{\feature(\state,\policy(\state))} }_q}.
\end{align}


\subsection{Connections to prior pessimistic learning rules}\label{sec:ub-discussion}
Now we present several connections to existing methods used for offline linear contextual bandits. 

\paragraph{Connection between $\estPolicy_2$ and Bellman-consistent pessimism.}
Xie et al.~\citep{xie2021bellman} proposed the idea of Bellman-consistent pessimism (\BCP{}) for solving offline reinforcement learning with general function approximation. When specialized to linear contextual bandits, the \BCP{} learning rule first forms a version space that includes all possible linear reward functions with small $\ell_2$ prediction error on the observed datasets. Then, \BCP{} defines each policy's pessimistic value as the smallest value the policy can achieve in the version space. Finally, \BCP{} returns the policy that has the highest pessimistic value. In fact, \BCP{} exactly matches the learning rule $\estPolicy_2$ proposed herein. To see this, it suffices to note that the empirical estimate of the Bellman error (cf.~Equation~(3.1) in the paper~\citep{xie2021bellman}) in the linear contextual bandit case is given by 
\begin{align*}
\frac{1}{n} \sum_{i=1}^{n} \inparen{ \phi (s_i, a_i)^\top \theta - r_i}^2 - \inf_{\theta \in \mathbb{R}^{d}} \ \frac{1}{n} \sum_{i=1}^{n} \inparen{ \phi (s_i, a_i)^\top \theta - r_i}^2 = \norm{\empCov^{1/2}(\theta - \thetaOLS)}_2^2.
\end{align*}
Therefore a parameter $\theta$ having a small Bellman error is equivalent to having a small $\ell_2$ distance to the OLS estimate.  
\citet{xie2021bellman} prove that \BCP{} enjoys the guarantee (up to log factors) of \version{$\sqrt{d/n} \cdot \E_{\state \sim \testDist} \big[ \| \empCov^{-1/2} \feature(\state, \optPolicy(\state)) \|_2\big]$,}{
\begin{align*}
  \val(\optPolicy) - \val(\estPolicy_{\mathsf{BCP}}) \lesssim \sqrt{\frac{d }{n}} \cdot \E_{\state \sim \testDist} \insquare{\norm{\empCov^{-1/2} \feature(\state, \optPolicy(\state))}_2},
\end{align*}}
which is loose compared to our theoretical guarantee $\sqrt{d / n} \ \cdot \ \|\empCov^{-1/2} \E_{\state \sim \testDist} \feature(\state, \optPolicy(\state))\|_{2}$\version{, as}{, since 
\[
\norm{\empCov^{-1/2} \E_{\state \sim \testDist} \feature(\state, \optPolicy(\state))}_{2} \leq \E_{\state \sim \testDist} \insquare{\norm{\empCov^{-1/2} \feature(\state, \optPolicy(\state))}_2},
\]}
a consequence of Jensen's inequality and the convexity of the $\ell_2$ norm.

Similar ideas using the $\ell_2$ confidence set also appear in a recent paper by ~\citet{zanette2021provable}; their actor-critic algorithm, PACLE, can be interpreted as providing a computationally efficient way to solve $\estPolicy_2$.\footnote{
While we focus on the statistical properties of the family $\{\estPolicy_p\}$ in this work, we believe that the actor-critic approach developed by \citet{zanette2021provable} can be extended to yield tractable algorithms for general $p\ge 1$.}

\paragraph{Connection between \name{} and lower confidence bound for tabular contextual bandits.} 
We now discuss how the \LCB{} learning rule for the tabular setting is a specialization of \name{}.
The tabular contextual bandit setting is a special case of the linear setting with the feature mapping $\feature(\state,\action) = e_{\state\action}$ (the canonical basis vector indexed by $(\state, \action)$). For notational convenience, we define $\numS \coloneqq \abs{\stateS}$, $\numA \coloneqq \abs{\actionS}$, $\empReward(\state, \action)$ to be the empirical average reward, and $n(\state,\action)$ to be the number of times the pair $(\state, \action)$ is seen in the dataset.

\citet{rashidinejad2021bridging} present the following lower confidence bound (\LCB{}) learning rule:
\version{
  \begin{align}\label{eq:lcb}
    \text{for each }s, \qquad \estPolicy_\mathsf{LCB}(\state) \coloneqq \argmax_{\action \in \actionS} \ \empReward(\state, \action) - \beta \cdot \sqrt{\frac{\log(SA/\confLevel)}{n(\state,\action)}}.
  \end{align}
}{
  \begin{align}\label{eq:lcb}
    \text{for each }s, \qquad \estPolicy_\mathsf{LCB}(\state) \coloneqq \argmax_{\action \in \actionS} \ \empReward(\state, \action) - \beta \cdot \sqrt{\frac{\log(SA/\confLevel)}{n(\state,\action)}}.
  \end{align}
}
In essence, the quantity \version{$\empReward(\state, \action) - \beta \cdot \sqrt{\frac{\log(SA/\confLevel)}{n(\state,\action)}}$}{$\empReward(\state, \action) - \beta \cdot \sqrt{\frac{\log(SA/\confLevel)}{n(\state,\action)}}$} acts as a lower confidence bound for the true mean reward $\reward(\state, \action)$. In every state, \LCB{} picks the action that maximizes this lower confidence bound. 
It is easy to verify that \LCB{}~\eqref{eq:lcb} exactly corresponds to \name{} (with proper choices of $\beta$); one just needs to check the max-only formulation in \Cref{eq:thetalp-estimator} with $p=\infty$ and $q=1$. 

In establishing performance guarantees for \LCB{}, \citet{rashidinejad2021bridging} assume that the covariates $\{(\state_i,\action_i)\}_{i=1}^n$ are drawn i.i.d.~from a behavior distribution $\behDist \in \Delta(\stateS \times \actionS)$ (as opposed to our fixed design setting). Nevertheless, it is straightforward to translate our results to this random design case by using Chernoff bounds.

\begin{corollary}\label{cor:tabular}
In the tabular setting, with probability at least $1-\confLevel$, the learning rule $\estPolicy_\infty$ with $\Theta$ given by \Cref{eq:thetalp} achieves the suboptimality:
\begin{align*}
  \val(\optPolicy) - \val(\estPolicy_\infty) \lesssim \sqrt{ \frac{ \log(SA/\confLevel) }{n} } \cdot \inparen{\sum_{s} \frac{\testDist(s)}{\sqrt{\behDist(\state, \optPolicy(\state))}}},
\end{align*}
as long as $n \gtrsim \log(\numS/\confLevel) \cdot \inparen{\min_s\{\behDist(\state, \optPolicy(s))\} }^{-1}$. 
\end{corollary}

\version{\Cref{cor:tabular} is proved in \Cref{sec:proof-cor-tabular}.}{
\noindent The proof of \Cref{cor:tabular} can be found in \Cref{sec:proof-cor-tabular}. \medskip}

\noindent Compared to the upper bound in the paper~\citep{rashidinejad2021bridging}, \Cref{cor:tabular} is more fine-grained, or ``problem-dependent'', as the suboptimality bound depends on the interaction between the specific behavior distribution $\behDist$ and test distribution $\testDist$. In contrast, \citet{rashidinejad2021bridging} consider the class of tabular instances with bounded single-policy concentrability coefficient. 

\begin{definition}\label{def:concentrability}
The single-policy concentrability coefficient is defined as $C^\star \coloneqq \sup_{\state \in \stateS} \frac{\rho(s)}{\mu(s,\pi^\star(s))}$.
\end{definition}

\noindent \Cref{cor:tabular} readily recovers the performance guarantee for \LCB{} of
\version{
$\tilde{O}(\sqrt{ { SC^\star } / {n} })$ 
}{
  \begin{align*} 
    \val(\optPolicy) - \val(\estPolicy_\infty) \lesssim \sqrt{ \frac{ SC^\star \log(SA/\confLevel) }{n} }
  \end{align*}
}
established in the paper~\citep{rashidinejad2021bridging}, which is optimal in the regime where $C^\star \ge 2$.

\paragraph{Connection to pessimistic value iteration.}
We give another example of how to interpret pessimistic learning rules using the idea of confidence set construction. Consider the pessimistic value iteration (\PEVI{}) learning rule proposed by~\citet{jin2021pessimism}. \PEVI{} directly extends \Cref{eq:lcb} to the linear setting:
\begin{align}\label{eq:pevi}
  \estPolicy_\mathsf{PEVI}(s) \coloneqq \argmax_{\action \in \actionS} \ \feature(\state,\action)^\top \thetaOLS - \beta \cdot \norm{\empCov^{-1/2} \feature(\state,\action)}_2,
\end{align}
where the right hand side still acts as a lower confidence bound for the true mean reward $r(s,a)$.
\PEVI{} bears striking resemblance with the max-only formulation~\eqref{eq:thetalp-estimator} (with $p=q=2$), with the key difference that the max-only formulation is ``averaged'' over the test distribution $\testDist$, while \PEVI{} directly discounts every $(\state,\action)$ pair. \PEVI{} does not immediately fit into our confidence set framework. However, if we modify the minimization over confidence sets to minimization over functionals $\theta: \stateS \to \R^d$, then we can rewrite \PEVI{} as
\version{
\begin{align*}
  \estPolicy_\mathsf{PEVI} &\coloneqq \argmax_{\pi \in \Pi} \inf_{\theta \in \Theta } \ \E_{\state \sim \testDist} [ \feature(\state, \policy(\state) )^{\top} \theta(s) ], \\
  &\text{where } \Theta = \inbraces{s\mapsto \theta(s) \mid \norm{ \empCov^{1/2}(\theta(s) - \thetaOLS) }_2 \le \beta, \text{ for all } s}.
\end{align*}}{
  \begin{align*}
    \estPolicy_\mathsf{PEVI} \coloneqq \argmax_{\pi \in \Pi} \inf_{\theta \in \Theta } \ \E_{\state \sim \testDist} [ \feature(\state, \policy(\state) )^{\top} \theta(s) ], \quad \text{where } \Theta = \inbraces{s\mapsto \theta(s) \mid \norm{\empCov^{1/2}(\theta(s) - \thetaOLS)}_2 \le \beta, \text{ for all } s}.
  \end{align*}
}
In other words, \PEVI{} enlarges the $\ell_2$ confidence set by separately picking a pessimistic parameter $\theta(s)$ for each state~$\state\in \stateS$. \citet{jin2021pessimism} prove the guarantee (up to log factors) of \version{$\sqrt{d^2/n} \cdot \E_{\state \sim \testDist} \insquare{\lVert \empCov^{-1/2} \feature(\state, \optPolicy(\state)) \rVert_2}$,}{\begin{align*}
  \val(\optPolicy) - \val(\estPolicy_{\mathsf{PEVI}}) \lesssim \sqrt{\frac{d^2}{n}} \cdot  \E_{\state \sim \testDist} \insquare{\norm{\empCov^{-1/2} \feature(\state, \optPolicy(\state))}_2},
\end{align*}}
which is loose due to the extra factor of $d$ and the interchanging of the expectation and the norm. However, their guarantee holds for all test distributions---as opposed to 
a fixed test distribution $\testDist$. This is a consequence of being pessimistic for every state $\state$.

\section{Which learning rule should one use?}\label{sec:lower-bounds}
Having introduced a general strategy for building pessimistic learning rules by constructing $\ell_p$ confidence sets, it is natural to ask which $\estPolicy_p$ one should use. To enable such comparisons, we investigate the statistical limits of offline linear contextual bandits over constrained sets of problem instances. 

\subsection{Minimax lower bound for constrained instances}
For any feature mapping $\feature: \stateS \times \actionS \to \R^d$, sample size $n \in \N$, and two quantities $q \in [1,\infty)$, $\Lambda > 0$, we define a set of linear contextual bandit (CB) instances\footnote{For brevity, we omit the dependence on $\feature$ and $n$ in the notation $\mathsf{CB}_q(\Lambda)$.} as follows: 
\version{
    \begin{align*}
        \mathsf{CB}_q(\Lambda) \!\! \coloneqq  \!\! \inbraces{( \testDist, \{(s_i, a_i)\}_{i=1}^n, \theta^\star , R) \mid \norm{ \empCov^{-1/2} \E_{\state \sim \testDist} [\phi(\state, \optPolicy(\state))] }_q \leq \Lambda,\ R \text{ is 1-subgaussian} }.
    \end{align*}
}{
\begin{align}\label{eq:CB-instances}
    \mathsf{CB}_q(\Lambda) \coloneqq \inbraces{( \testDist, \{(s_i, a_i)\}_{i=1}^n, \theta^\star , R) \mid \norm{ \empCov^{-1/2} \E_{\state \sim \testDist} [\phi(\state, \optPolicy(\state))] }_q \leq \Lambda,\ R \text{ is 1-subgaussian} }.
\end{align}
}
The set $\mathsf{CB}_q(\Lambda)$ includes all the linear contextual bandit instances such that a sort of ``complexity measure'' 
$\complexity_q \coloneqq \lVert \empCov^{-1/2} \E_{\state \sim \testDist} [\phi(\state, \optPolicy(\state))] \rVert_q$ is at most $\Lambda$. 
Our motivation to consider the rate of estimation in the CB family $\mathsf{CB}_q(\Lambda)$ are 
two-fold. 
First, in view of \Cref{thm:lp-estimator-bound}, the family $\mathsf{CB}_q(\Lambda)$ admits a good learning rule, specifically $\estPolicy_p$ with $1/p + 1/q = 1$, since for every $\calQ \in \mathsf{CB}_q(\Lambda)$, w.p.~at least $1-\delta$,
\version{
    \begin{align}\label{eq:upper-for-constrained-class}
        \val_\CBinstance^\star - \val_\CBinstance(\estPolicy_p) \lesssim d^{1/p} \sqrt{ { \log(d/\confLevel)} / {n} } \cdot \Lambda,
    \end{align}
}{
    \begin{align}\label{eq:upper-for-constrained-class}
        \val_\CBinstance^\star - \val_\CBinstance(\estPolicy_p) \lesssim d^{1/p} \sqrt{ \frac{ \log(d/\confLevel)}{n} } \cdot \Lambda,
    \end{align}
}
where $\val_{\CBinstance}^\star$ denotes the optimal value in instance $\calQ$ and $\val_\CBinstance(\policy)$ denotes the value of policy $\policy$ in instance $\calQ$. Thus, it is natural to view $\mathfrak{C}_q$ 
as a certain complexity measure for offline learning in linear contextual bandits.  
Second, prior work~\citep{zanette2021provable,rashidinejad2021bridging,yin2021towards} has proven various types of lower bounds on offline learning using either the $\ell_2$ quantity $\complexity_2$ or the $\ell_1$ quantity $\complexity_1$.  We will elaborate more on this point later.

Now we are ready to present the minimax lower bounds for these families of CB instances.



\begin{theorem}\label{thm:lp-lower-bound}
    For every $d \geq 2$, there exists a feature mapping $\feature$ such that the following lower bound holds. For any $p,q \geq 1$ such that $1/p + 1/q = 1$, as long as $\Lambda \ge \sqrt{8} \cdot d^{1/q - 1/2}$ and $n \ge d^{2/p} \Lambda^2$, we have
    \version{
        \begin{align*}
            \inf_{\estPolicy} \sup_{\CBinstance \in \mathsf{CB}_q(\Lambda)} \E
            [\optVal_{ \CBinstance } - \val_{ \CBinstance }( \estPolicy ) ] \ge c \cdot  d^{1/p} \sqrt{ 1/n } \cdot \Lambda,
        \end{align*}
    }{
    \begin{align*}
        \inf_{\estPolicy} \sup_{\CBinstance \in \mathsf{CB}_q(\Lambda)} \E
        [\optVal_{ \CBinstance } - \val_{ \CBinstance }( \estPolicy ) ] \ge c \cdot  d^{1/p} \sqrt{ \frac{1}{n} } \cdot \Lambda,
    \end{align*}}
    where $c > 0$ is some universal constant. Furthermore, when $p=\infty, q=1$, the lower bound holds for the extended range of $\Lambda \ge 2$.
\end{theorem}

\version{The proof can be found in \Cref{sec:improved-lower-bound}. It relies on a reduction to a bound for the minimax regret of the multi-armed bandit problem.}{\noindent The proof can be found in \Cref{sec:improved-lower-bound}. It relies on a reduction to a bound for the minimax regret of the multi-armed bandit problem.\medskip\noindent}

\noindent We note that \Cref{thm:lp-lower-bound} also consists of a family of lower bounds for each $\ell_q$ norm constrained CB class. 
By comparing the lower bound in \Cref{thm:lp-lower-bound} with the upper bound~\eqref{eq:upper-for-constrained-class} obtained by $\estPolicy_p$, we see that for the $\ell_q$ norm constrained class $\mathsf{CB}_q(\Lambda)$, the learning rule $\estPolicy_p$ with $1/p + 1/q = 1$ is minimax rate-optimal, up to a $\log d$ factor. For instance, over the $\ell_2$ class $\mathsf{CB}_2(\Lambda)$, the minimax rate of estimation is $\tilde{\Theta}(\sqrt{ d / n } \cdot \Lambda) $, while over the $\ell_1$ class $\mathsf{CB}_1(\Lambda)$, the rate is given by $\tilde{\Theta}(\sqrt{1 / n } \cdot \Lambda)$.

On a technical front, it would be interesting to extend \Cref{thm:lp-lower-bound} to the entire range of $\Lambda \ge 0$. It is unclear whether the same minimax rate of $\Omega(d^{1/p}/\sqrt{n} \cdot \Lambda)$ holds when $\Lambda = O(d^{1/q-1/2})$, or whether we can achieve faster rates in the small $\Lambda$ regime. In the tabular setting, \citet{rashidinejad2021bridging} recently showed that \LCB{} achieves fast $1/n$ rates when the single policy concentrability coefficient is small; similar results might hold in the linear setting. Several limitations prevent us from extending the range of $\Lambda$ in \Cref{thm:lp-lower-bound}; \Cref{sec:lb-limitations} provides more technical details.

\subsection{Adaptive minimax optimality of \name{}}\label{sec:adaptive-minimax-opt}

We point out a even stronger message delivered in \Cref{thm:lp-lower-bound}: \name{} is \emph{adaptively minimax optimal} for solving the offline linear contextual bandit problem. This is illustrated in \Cref{fig:adaptive-opt}, where we plot the sample complexity $n$ required in order to achieve constant suboptimality (say, 0.01) for various $\mathsf{CB}_{p/(p-1)}(\Lambda)$. (For sake of  illustration, it is more convenient to work with $p$ rather than $q$ on the $x$-axis.)

\version{\begin{wrapfigure}[19]{r}{0.45\linewidth}
    \centering
    \vspace{-1em}
    \includegraphics[scale=0.19, trim={0cm 19.5cm 35cm 0cm},clip]{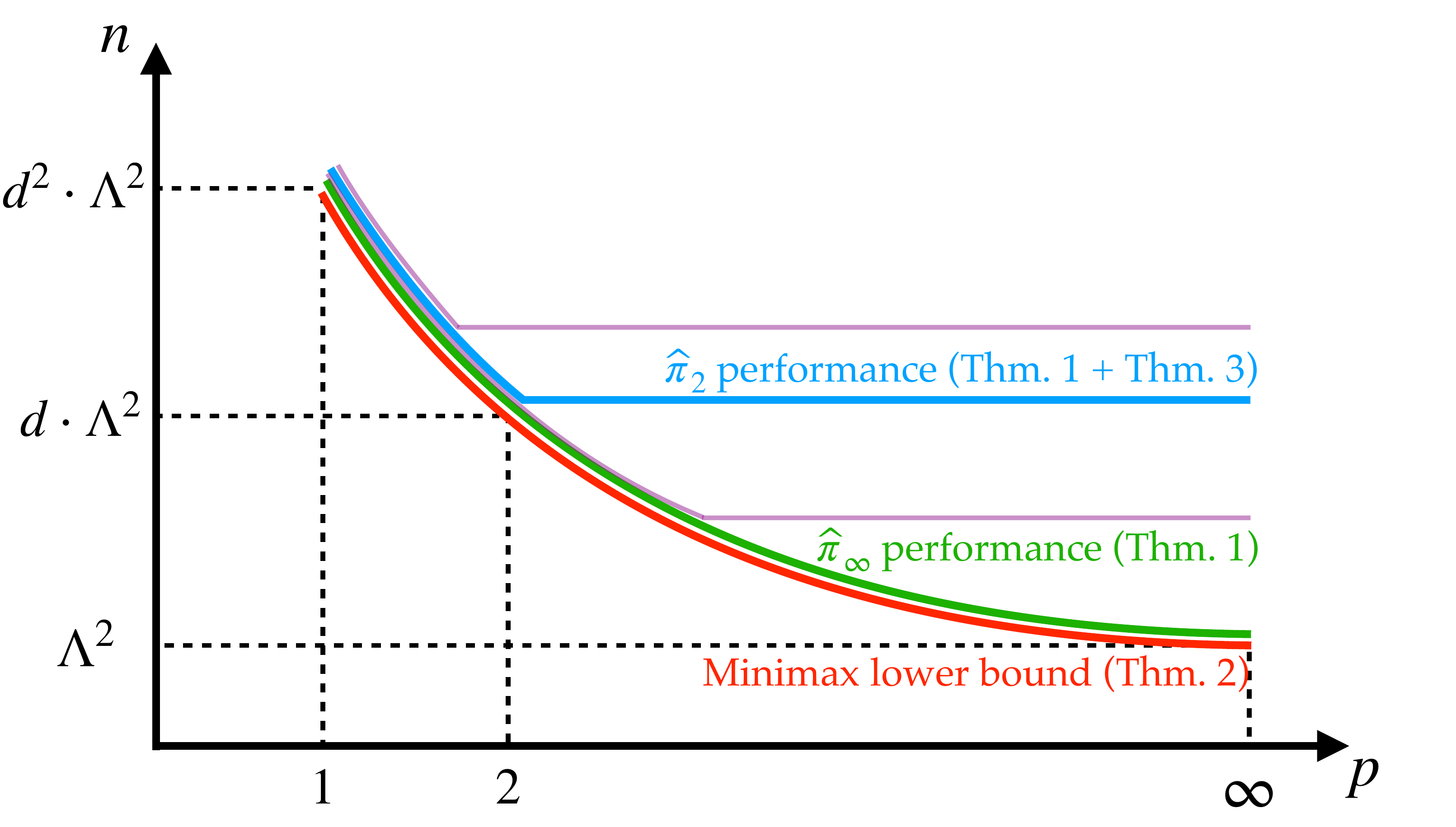}
    \caption{Sample complexity of $\estPolicy_{\widetilde{p}}$ (for various $\widetilde{p}$) over different $\mathsf{CB}_{p/(p-1)}(\Lambda)$ classes. The red line corresponds the minimax lower bound. Other lines correspond to different values of $\widetilde{p}$ and show the number of samples $n$ required to ensure $\sup_{\calQ \in \mathsf{CB}_{p/(p-1)} (\Lambda)} \E[\val_{\CBinstance}^\star- \val_{\CBinstance}(\estPolicy)] \le 0.01$. The blue and green lines correspond to $\estPolicy_2$ and $\estPolicy_\infty$ respectively.  Two purple lines correspond to $\estPolicy_{\widetilde{p}}$ for some $\widetilde{p}\in(1,2)$ and $\widetilde{p} \in (2,\infty)$. \name{} attains minimax optimality over every class, while other $\estPolicy_{\widetilde{p}}$ do not.}\label{fig:adaptive-opt}
\end{wrapfigure}}{\begin{figure}
    \centering
    \includegraphics[scale=0.28, trim={0cm 19cm 35cm 0cm},clip]{adaptive_opt.pdf}
    \caption{Sample complexity of $\estPolicy_{\widetilde{p}}$ (for various $\widetilde{p}$) over different $\mathsf{CB}_{p/(p-1)}(\Lambda)$ classes. The red line corresponds the minimax lower bound. Other lines correspond to different values of $\widetilde{p}$ and show the number of samples $n$ required to ensure $\sup_{\calQ \in \mathsf{CB}_{p/(p-1)} (\Lambda)} \E[\val_{\CBinstance}^\star- \val_{\CBinstance}(\estPolicy)] \le 0.01$. The blue and green lines corresponds to $\estPolicy_2$ and $\estPolicy_\infty$ respectively.  Two purple lines correspond to $\estPolicy_{\widetilde{p}}$ for some $\widetilde{p}\in(1,2)$ and $\widetilde{p} \in (2,\infty)$. \name{} attains minimax optimality over every class, while other $\estPolicy_{\widetilde{p}}$ do not.}
    \label{fig:adaptive-opt}
\end{figure}}

As indicated by the red line, \Cref{thm:lp-lower-bound} shows that every learning rule must incur sample complexity at least $\Omega(d^{2/p} \Lambda^2)$. Likewise, we can also follow the discussion after \Cref{thm:lp-estimator-bound} to see that the performance upper bound of $\estPolicy_\infty$ is
\version{$d^{1/p} \sqrt{ \tfrac{ \log(d/\confLevel) }{n} } \cdot \complexity_q$ for all $p, q \geq 1, 1/p + 1/q = 1$.}{\begin{align*}
\val(\optPolicy) - \val(\estPolicy_{\infty}) &\lesssim \sqrt{ \frac{ \log(d/\confLevel) }{n} } \cdot \norm{\empCov^{-1/2} \E_{\state \sim \testDist} \insquare{ \feature( \state, \optPolicy(\state) )}}_1 \\
&\leq d^{1/p} \sqrt{ \frac{ \log(d/\confLevel) }{n} } \cdot \norm{\empCov^{-1/2} \E_{\state \sim \testDist} [\phi(\state, \optPolicy(\state))] }_q, \: \text{for all } p, q \geq 1, 1/p + 1/q = 1.
\end{align*}}
Thus, \name{} attains the green line in \Cref{fig:adaptive-opt}; that is, \name{} is \emph{simultaneously} minimax rate-optimal for all $\ell_q$-norm constrained classes $\mathsf{CB}_q(\Lambda)$, up to a $\log d$ factor.\footnote{We did not investigate when the $\log d$ factor in \Cref{thm:lp-estimator-bound} can be removed, so for example, it is possible that $\estPolicy_2$ beats $\estPolicy_\infty$ over $\mathsf{CB}_2(\Lambda)$ by a factor of $\sqrt{\log d}$.} From worst-case perspective, one should always prefer using \name{} given an unknown CB instance.

Is this adaptive optimality property unique to \name{} among the family $\{\estPolicy_p\}_{p\ge 1}$ we consider? Below, we answer this question in the positive by presenting a separation result.



\begin{theorem}[Informal]\label{thm:separation-l2-linfty-new} Fix any $p \ge 1$. For sufficiently large $n, d$, there exists a contextual bandit instance $\CBinstance \in \mathsf{CB}_1(\Lambda)$ with $\Lambda = \sqrt{8d}$, such that with probability at least $1/4$, $\estPolicy_p$ has suboptimality at least $\Omega(d^{1/p}/\sqrt{n} \cdot \Lambda)$. 
\end{theorem}

\noindent Since \name{} attains a suboptimality of $\tilde{O}(1/\sqrt{n} \cdot \Lambda)$ over the class $\mathsf{CB}_1(\Lambda)$, \Cref{thm:separation-l2-linfty-new} shows that every other $\estPolicy_p$ is \emph{suboptimal} over the class $\mathsf{CB}_1(\Lambda)$.

\version{}{\medskip}

\noindent A formal statement of \Cref{thm:separation-l2-linfty-new} and its proof can be found in \Cref{sec:separations}. The key intuition in the proof is that the $\ell_p$ confidence sets capture a notion of error that is ``averaged'' over all directions, while the $\ell_\infty$ confidence sets separately estimate the error in each direction. In the hard instance we construct, only one direction determines the difficulty of the offline learning problem, so $\estPolicy_p$ is worse by a factor of $d^{1/p}$. There is nothing special about the choice $\Lambda = \sqrt{8d}$, and our construction works for any $\Lambda \ge \Omega(\sqrt{d})$; we pick it to enable comparison with \Cref{thm:lp-lower-bound}.

For sake of discussion, consider $\estPolicy_2$. \Cref{thm:lp-estimator-bound} shows that $\estPolicy_2$ attains the rate:
\begin{align*}
\val(\optPolicy) - \val(\estPolicy_{2}) 
&\lesssim \begin{cases}
d^{1/p} \sqrt{  \frac{\log(d/\confLevel) }{n} } \cdot \norm{\empCov^{-1/2} \E_{\state \sim \testDist} \insquare{ \feature( \state, \optPolicy(\state) )}}_q &\text{when } q \ge 2,\\
\sqrt{ \frac{d \log(d/\confLevel) }{n} } \cdot \norm{\empCov^{-1/2} \E_{\state \sim \testDist} \insquare{ \feature( \state, \optPolicy(\state) )}}_q &\text{when } q \in [1,2].
\end{cases}
\end{align*}
Together, \Cref{thm:lp-lower-bound} and \ref{thm:separation-l2-linfty-new} provide the message that both cases in the upper bound are tight (up to log factors).
In the range $p \in [1,2]$ (or $q \ge 2$), \Cref{thm:lp-lower-bound} shows that $\estPolicy_2$ attains the minimax optimal rate (up to log factors)  over $\mathsf{CB}_{p/(p-1)}(\Lambda)$, i.e.,~it is adaptively minimax optimal here. This explains the curved part of the blue line in \Cref{fig:adaptive-opt}. On the other hand, \Cref{thm:separation-l2-linfty-new} shows that $\estPolicy_2$ cannot obtain the minimax rate over $\mathsf{CB}_1(\Lambda)$. Instead, $\estPolicy_2$ may require $\Omega(d \cdot \Lambda^2)$ samples in order to achieve constant suboptimality. Since for any $p$, the set $\mathsf{CB}_{p/(p-1)}(\Lambda) \supseteq \mathsf{CB}_1(\Lambda)$, we know that $\estPolicy_2$ may require $\Omega(d\cdot \Lambda^2)$ samples for any $\mathsf{CB}_{p/(p-1)}(\Lambda)$. Thus, the second case is tight when $p \ge 2$ (or $q\in[1,2]$), explaining the flat part of the blue line in \Cref{fig:adaptive-opt}. In general, for any finite $\widetilde{p}$, the learning rule $\estPolicy_{\widetilde{p}}$ will be adaptively optimal for $\mathsf{CB}_{p/(p-1)}(\Lambda)$ only in the range $p\in[1,\widetilde{p}]$, and afterwards the sample complexity will ``flatten out'', as illustrated by the purple lines in \Cref{fig:adaptive-opt}.

\begin{figure}[t]
    \centering
    \subfigure[Random rotation]{\includegraphics[width=0.32\linewidth]{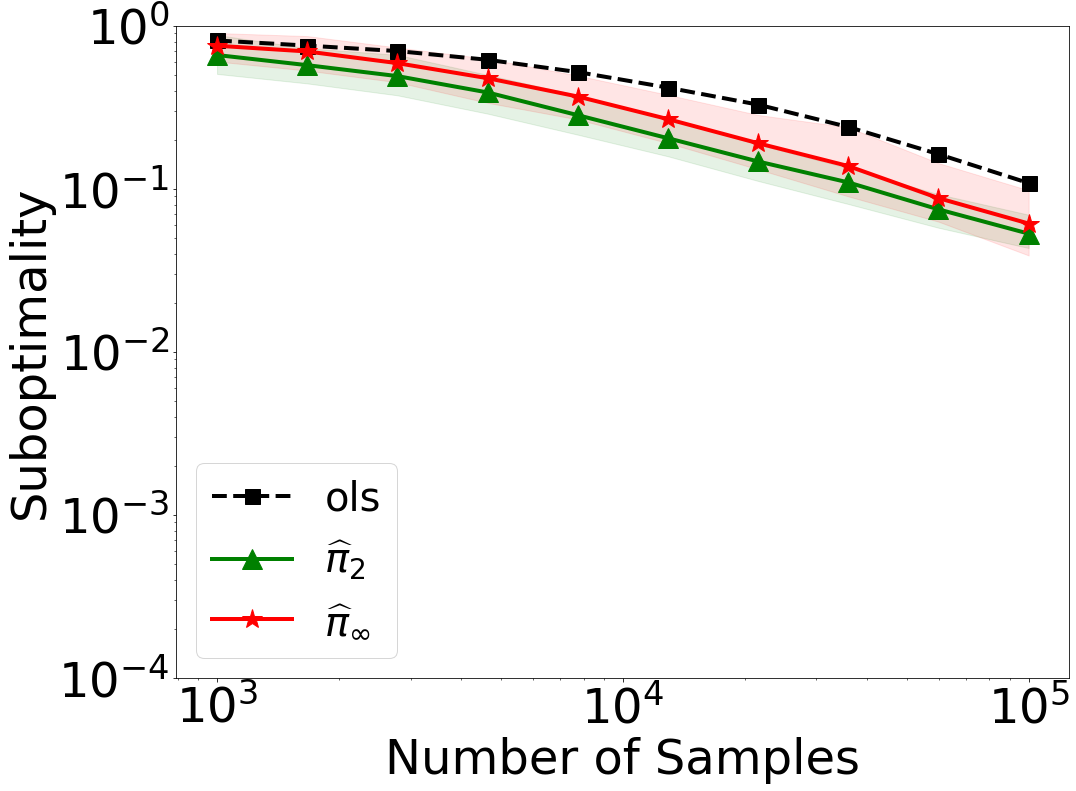}\label{fig:nonidrr}}
    \hfill
    \subfigure[Basis-aligned]{\includegraphics[width=0.32\linewidth]{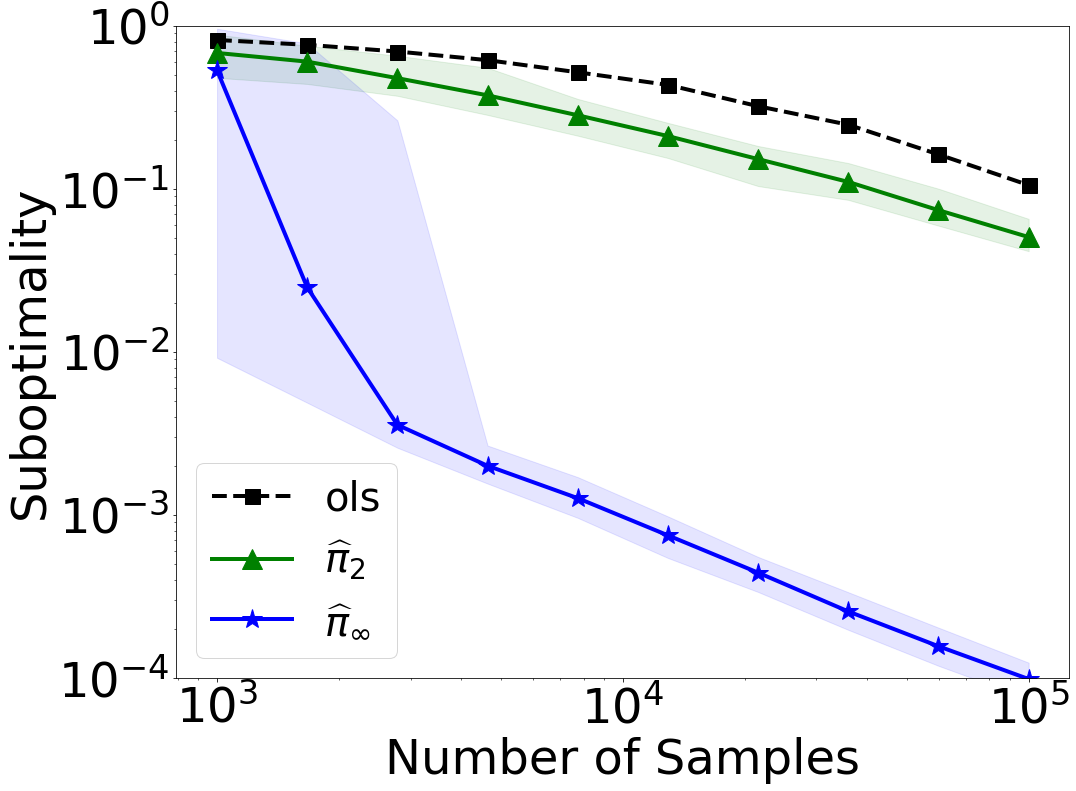}\label{fig:nonidbb}}
    \hfill
    \subfigure[$\ell_1$ and $\ell_2$ complexities]{\includegraphics[width=0.32\linewidth]{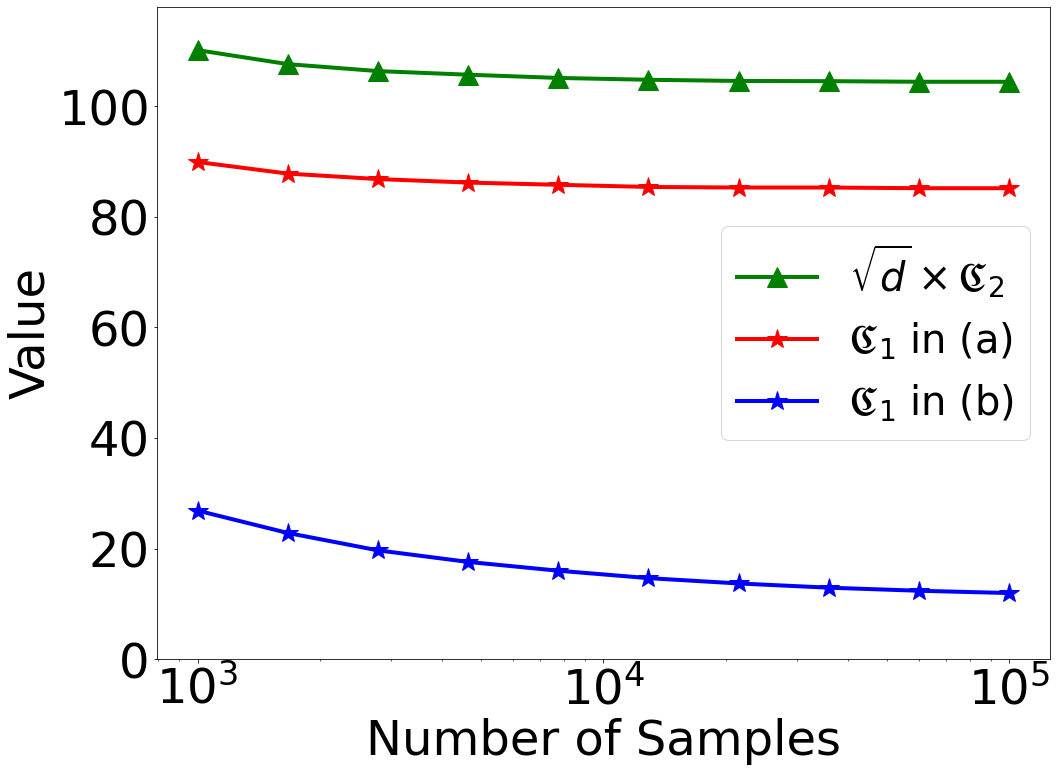}\label{fig:nonidbb}}
    \caption{Comparing the performance of the plug-in rule, $\estPolicy_2$, and $\estPolicy_\infty$ on linear contextual bandit instances with $d=100$, averaged over 100 trials, with 90\% confidence intervals. (a) $\phi_i \sim \calN(0, Q D Q^\top)$ and $\theta^\star = Qe_{20}$, where $Q$ is a random rotation matrix and $D$ is a diagonal matrix with entries $D_{ii} = i^{-1}/(\sum_i i^{-1})$. (b) $\phi_i \sim \calN(0, D)$ and $\theta^\star = e_{20}$. (c) computed average values for $\complexity_1$ and $\sqrt{d}\times \complexity_2$. The quantity $\complexity_2$ is identical in both plots (a) and (b). For (a), $\complexity_1 \approx \sqrt{d} \times \complexity_2$, while for (b), $\complexity_1 \ll \sqrt{d}\times \complexity_2$.}
    \label{fig:unitball}
    \version{\vspace*{-1.5em}}{}
\end{figure}

\paragraph{Experimental Evidence.} In order to further validate this claim, we provide experimental evidence which shows that $\estPolicy_2$ does not adapt to ``easy'' CB instances. In \Cref{fig:unitball}, we consider a simple offline linear contextual bandit in which there is a single state and the feature set is $B_2^d$; thus the policy learning problem is equivalent to finding a vector $\pi \in \mathbb{S}^{d-1}$ that maximizes $V(\pi) \coloneqq \pi^\top \theta^\star$. We vary the offline dataset distribution and the hidden parameter $\theta^\star$. When $\theta^\star$ is basis-aligned, we have $\complexity_1 \ll \sqrt{d}\times \complexity_2$; when $\theta^\star$ is not basis-aligned, the two quantities are on the same order.

\subsection{Comparisons with prior lower bounds}
There exist several lower bound results for offline reinforcement learning in the literature. In this section, we compare our lower bounds (cf.~\Cref{thm:lp-lower-bound}) with the prior bounds and highlight several improvements offered by our results. 

\paragraph{Comparison to lower bounds w.r.t.~a single $\Lambda$.}

Our lower bounds are stronger than those provided in the papers~\citep{jin2021pessimism, zanette2021provable}, which hold for specific choices of $p=q=2$ and a single fixed $\Lambda$. Take Theorem 2 of \citet{zanette2021provable} for example. Zanette et al.~proved that the minimax rate of estimation over $\mathsf{CB}_2(\Lambda = d)$ is given by $d^{3/2} / \sqrt{n}$. Such a lower bound fails to uncover the fundamental scaling on the complexity $\Lambda$.\footnote{For instance, their result does not preclude the possibility that the correct lower bound over $\mathsf{CB}_2(\Lambda)$ takes an expression, say, $d^{-98.5} \Lambda^{100}/\sqrt{n}$.} Theorem 4.7~of \citet{jin2021pessimism} shows a result in similar spirit; their construction essentially shows a minimax lower bound of $1/\sqrt{n}$ over $\mathsf{CB}_2(\Lambda)$ when $\Lambda = \Theta(1)$. Furthermore, their lower bound is loose by a factor of $\sqrt{d}$ since they reduce to a two-point hypothesis testing problem. In contrast, our lower bound holds for nested families of CB instances with \emph{varying} complexities $\Lambda$\version{}{ (cf.~Equation~\eqref{eq:CB-instances})}, which better showcases that the norm quantity is an intrinsic measure of difficulty for offline learning. 

\paragraph{Connections with single-policy concentrability.}
\version{
}{
\begin{figure}[]
    \centering
    \includegraphics[scale=0.16, trim={0cm 10cm 9cm 0cm},clip]{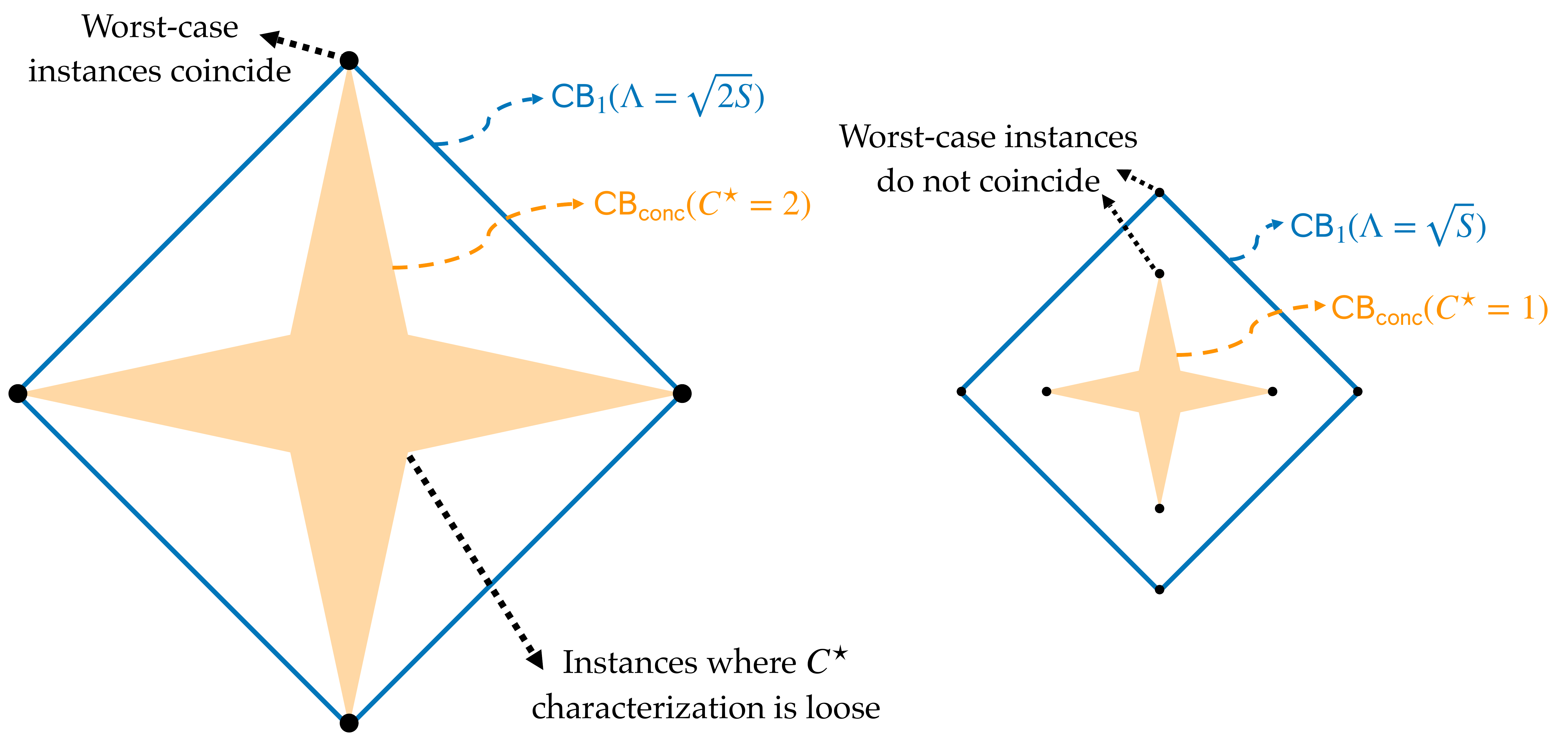}
    \caption{An illustration of the relationship between single policy concentrability and boundedness of $\complexity_1$. Left: When $C^\star = 2$, the complexity $\complexity_1$ always provides a tighter characterization of the problem difficulty, and the worst-case instances coincide. Right: When $C^\star = 1$, the complexity $\complexity_1$ does not provide a tighter characterization (since the minimax rate is $\sqrt{S/n}$); on the other hand, the minimax rate over the $C^\star$ class is the smaller $S/n$.}
    \label{fig:concentrability}
\end{figure}
}
Our lower bound shares a similar flavor as that established in the paper~\citep{rashidinejad2021bridging}, with the key difference lying in the class of CB instances under consideration: \citet{rashidinejad2021bridging} consider the contextual bandit instances $\mathsf{CB}_{\mathrm{conc}}(C^\star)$ with bounded single-policy concentrability coefficient $C^\star$ (cf.~\Cref{def:concentrability}), while we consider the instances with bounded complexity $\complexity_1$. These two quantities are intimately related, and we illustrate the relationship in \Cref{fig:concentrability}. As we have alluded to in \Cref{sec:ub-discussion}, one has the inclusion 
\begin{align*}
\mathsf{CB}_{\mathrm{conc}}(C^\star) \subseteq \mathsf{CB}_1(\sqrt{SC^\star}).
\end{align*}
When $C^\star \ge 2$, the minimax rate of estimation over $\mathsf{CB}_{\mathrm{conc}}(C^\star)$ exactly matches that over $\mathsf{CB}_1(\sqrt{SC^\star})$, which implies that the hard instances for $\mathsf{CB}_{\mathrm{conc}}(C^\star)$ are also the hard instances in $\mathsf{CB}_1(\sqrt{SC^\star})$.
However, this no longer holds when $C^\star \in [1,2)$. Take $C^\star = 1$ as an example. Rashidinejad et al.~show that the optimal rate over $\mathsf{CB}_{\mathrm{conc}}(C^\star=1)$ is $S / n$, while \Cref{thm:lp-lower-bound} indicates that the optimal rate over $\mathsf{CB}_1(\Lambda = \sqrt{S})$ is $\sqrt{S / n}$. There is no contradiction, since the hard instances we construct for $\mathsf{CB}_1(\sqrt{S})$ satisfy $C^\star \ge 2$. This shows that when $C^\star < 2$, we ``lose something'' by working with the larger $\mathsf{CB}_1(\sqrt{SC^\star})$ class, as we are no longer able to achieve the fast rates possible over $\mathsf{CB}_{\mathrm{conc}}(C^\star)$.

On the flip side, the quantity $\complexity_1$ can give tighter suboptimality guarantees than the $C^\star$ bound for a given instance. Consider the tabular instance where $\testDist = \mathrm{Unif}(\stateS)$ and $\mu(1, \optPolicy(1)) = 1/S^3$, while $\mu(\state, \optPolicy(\state)) = 1/S$ for all $s\ge 2$. This instance has $C^\star = S^2$, implying a guarantee of $S^{3/2}/\sqrt{n}$, while $\complexity_1 = O(\sqrt{S})$, implying a better guarantee of $\sqrt{S/n}$.

\version{
    \begin{figure}[t]
        \centering
        \includegraphics[scale=0.14, trim={0cm 10cm 8cm 0cm},clip]{c_star_comparison.pdf}
        \caption{Illustrating the relationship between single policy concentrability and boundedness of $\complexity_1$. Left: When $C^\star = 2$, the quantity $\complexity_1$ always provides a tighter characterization of the problem difficulty, and the worst-case instances coincide. R: When $C^\star = 1$, the quantity $\complexity_1$ does not provide a tight characterization in general.}
        \label{fig:concentrability}
        \vspace*{-1.5em}
    \end{figure}
}{
}

\subsection{A better complexity measure?}\label{sec:better-complexity}
Our results lend support to the claim that we should always use \name{}, since it is simultaneously minimax rate-optimal over all the $\ell_q$ norm-constrained contextual bandit classes. Furthermore, the $\ell_1$ quantity $\complexity_1$ can be thought of as a ``complexity measure'' that dominates other $\ell_q$ ``complexity measures'' $\complexity_q$ for $q > 1$. To see this, consider the following thought experiment. Suppose before solving the linear contextual bandit problem, an oracle told us that the instance satisfies $\complexity_q \le \Lambda$. The results herein show that we do not lose anything by assuming that the instance satisfies the weaker condition $\complexity_1 \le d^{1/p} \Lambda$; using \name{} will give us the optimal rate of $d^{1/p}/\sqrt{n} \cdot \Lambda$.

However this is certainly not the complete answer to guiding question of ``which pessimistic learning rule should one use for offline linear contextual bandits?''. One piece of evidence comes from the comparison with the single policy concentrability assumption: in the regime where $C^\star \in [1,2)$, we do ``lose something'' when we assume the instance satisfies the weaker condition $\complexity_1 \le \sqrt{SC^\star}$. Below we discuss another drawback associated with using $\complexity_1$ as the complexity measure.

\paragraph{Rotation ambiguity. }
One drawback of the complexity $\complexity_1$ (as well as the learning rule \name{}) lies in the fact that it is not rotation invariant. (In fact, $\complexity_2$ is the only rotational invariant complexity!) 
To see this, let $U \in \mathbb{R}^{d \times d}$ be a fixed rotation matrix. Suppose that the features are rotated from $\phi$ to $U \phi$, which yields a different $\ell_1$ complexity 
\version{$\complexity_1(U) \coloneqq \lVert U \empCov^{-1/2} \E_{\state \sim \testDist} \insquare{ \feature( \state, \optPolicy(\state) )}\rVert_1$,}{\begin{align*} \complexity_1(U) \coloneqq 
\norm{U \empCov^{-1/2} \E_{\state \sim \testDist} \insquare{ \feature( \state, \optPolicy(\state) )}}_1,
\end{align*}}
where $\empCov$ is defined using the old feature mapping. Since the $\ell_1$ norm is not rotation invariant, the $\complexity_1(U)$ varies for differing choices of $U$, by up to a $\sqrt{d}$ factor. Thus, we cannot claim that any ``complexity measure'' $\complexity_1(U)$ dominates others. A naive attempt to modify the $\ell_1$ set to be rotationally invariant by taking a minimization over $U$ also fails; observe that:
\begin{align*}
    \Theta_1^\mathrm{min} \coloneqq \inbraces{ \theta \in \R^{d} \mid  \inf_{U} \norm{U\empCov^{1/2}(\theta - \thetaOLS)}_1 \le \beta } = \inbraces{ \theta \in \R^{d} \mid \norm{\empCov^{1/2}(\theta - \thetaOLS)}_2 \le \beta } \eqqcolon \Theta_2,
\end{align*}
that is, we recover $\estPolicy_2$. A similar equivalence holds if we take the $\max$ inside the confidence set; we will recover the confidence set with an extra factor of $\sqrt{d}$.

\version{}{
However, it may be possible to leverage rotation ambiguity to design better heuristics for offline RL. In particular, one can develop iterative algorithms which first learn a good rotation $\hat{U}$ for the data, then solve the $\ell_\infty$ pessimism problem. We leave this to future work.}

\paragraph{Instance-dependent optimality?} 
Arguably, the strongest possible support for \name{} would be an instance-dependent lower bound which shows that for \emph{every specific} linear contextual bandit instance, the performance achieved by \name{} is not improvable. Instance-dependent optimality results have been shown for related problems such as policy evaluation~\citep{pananjady2020instance,khamaru2020temporal} and optimal value estimation~\citep{khamaru2021instance} in tabular MDPs; the recent work \cite{foster2021statistical} also employs the local minimax method for online bandit and RL problems. For offline bandits, the paper \cite{xiao2021optimality} shows how a particular definition of instance optimality cannot be achieved by any algorithm. Establishing instance-dependent guarantees for offline learning is an important direction for future research. 

Recent work~\citep{yin2021towards} establishes the local minimax rate for offline learning in terms of the complexity $\complexity_1$ for tabular contextual bandits. However, the theorem seems incorrect; we provide a counterexample in \Cref{sec:counterexample} to demonstrate---via explicit construction---that the complexity $\complexity_1$ cannot characterize the local minimax risk for a two-armed bandit instance. The key observation is that the reduction used in the proof of the paper \cite{yin2021towards} from offline policy learning to optimal value estimation is invalid; if the gap in rewards for different actions is large, offline policy learning is fundamentally easier than optimal value estimation. This in turn allows us to break the claimed parametric $1 / \sqrt{n}$ rate.


\section{Conclusion}\label{sec:conclusion}

In this paper, we introduce a family $\{\estPolicy_p\}_{p \ge 1}$ of pessimistic learning rules that include a number of prior works as special cases for the problem of offline learning in linear contextual bandits. We prove upper bounds for each learning rule $\estPolicy_p$ and show matching minimax lower bounds over appropriately defined constrained instance classes. Our results highlight the benefits of using \name{}, the $\estPolicy_\infty$ learning rule: namely (1) the guarantee for \name{} dominates all others; (2) \name{} is the sole learning rule with an adaptive minimax property. In particular, our results demonstrate that prior learning rules based on $\ell_2$ pessimism can be suboptimal (by a factor of $\sqrt{d}$).

Below we list several interesting directions for future investigation. 
\version{\vspace{-\topsep}}{}
\version{\begin{itemize}[leftmargin=0.5cm]}{\begin{itemize}}
    \item \emph{Extending to MDPs.} The MDP setting is more difficult due to the long horizon and transition dynamics. Extending the results of this paper to the MDP setting is an interesting future direction. One possible approach is to modify the PACLE algorithm \cite{zanette2021provable} to solve for any $\ell_p$ learning rule. 
    \item \emph{Gap-dependent bounds.} In online RL, there is a wealth of results which adapt to easy instances which are characterized by gap structure in the rewards, see, e.g., \cite{foster2020instance, wagenmaker2022instance}. Obtaining tight gap-dependent bounds for the offline setting is an interesting direction for future work.
    \item \emph{Offline RL with general function approximation.} In this paper, we focus on offline RL with linear function approximation. What is the right extension of these $\ell_p$ learning rules to general function approximation? While $\estPolicy_2$ has the natural interpretation of defining a version space with small squared prediction error, no such interpretation exists for \name{}. It would be interesting to establish an analog for \name{} for general function classes.
\end{itemize}

\version{}{\subsection*{Acknowledgements}
This work is supported by funding from the Institute for Data, Econometrics, Algorithms, and Learning (IDEAL).}

\newpage
\bibliographystyle{plainnat}
\bibliography{lcb_bounds_arxiv}

\newpage
\appendix
\section{Additional literature}\label{sec:related-work}
In this section, we mention a few additional related works which are not discussed in the main text.

First, we mention that offline contextual bandits have an extensive history dating back to early work in recommender systems \cite{li2010contextual, bottou2013counterfactual}. A popular approach is to use importance-reweighting to estimate the values of policies \cite{strehl2010learning, bottou2013counterfactual, swaminathan2015counterfactual, swaminathan2015self}. A key aspect of this approach is that it requires one to estimate or know the behavior policy which generated the offline data. In contrast, our work relies on the principle of pessimism and directly bounds the suboptimality in terms of the quality of the data coverage.

We highlight additional works which employ the principle of pessimism in order to address the issue of dataset coverage in offline RL. In the tabular setting, several additional works study the single policy concentrability assumption \cite{xie2021policy,shi2022pessimistic,yin2021towards, yan2022efficacy, li2022settling}. Several additional works study offline RL with function approximation, e.g.,~\cite{yin2022near, uehara2021pessimistic, uehara2021representation, cheng2022adversarially, zhan2022offline, chen2022offline}.

We also highlight several lower bounds for the offline RL problem. Most relevant to our work are lower bounds for the single policy concentrability assumption in the tabular setting \cite{rashidinejad2021bridging, xie2021policy}. In addition, there is a line of work for offline RL with function approximation which studies the interplay between data coverage and representation power \cite{chen2019information, foster2021offline, zanette2021exponential, wang2020statistical}. These results generally show lower bounds with exponential dependence on $H$. Since we focus on the linear contextual bandit setting ($H=1$), these results are tangential to our discussion.

\section{Proofs for \Cref{sec:pess-est}}

This section gathers the proofs for the results in~\Cref{sec:pess-est}.

\subsection{Proof of \Cref{lemma:valid-lp}}\label{sec:pf-lp-ub}

We decompose the proof into two steps: (1) proving $\theta^\star \in \thetalp$, and (2) proving $\sup_{\theta \in \Theta} \norm{\theta^\star - \theta} \le \beta$.

\paragraph{Step 1: proving $\theta^\star \in \thetalp$.}

It is easy to see that $\empCov^{1/2}(\theta^\star - \thetaOLS) = \frac{1}{\sqrt{n}} (\datamatrix^\top \datamatrix)^{-1/2} \datamatrix^\top \errorvec$, where $\errorvec \in \R^n$ has $i$-th entry equal to $\reward_i - \E[R(\state_i, \action_i)]$. Since each row of $(\datamatrix^\top \datamatrix)^{-1/2} \datamatrix^\top$ has unit $\ell_2$ norm, $[(\datamatrix^\top \datamatrix)^{-1/2} \datamatrix^\top \errorvec]_j$ is a 1-subgaussian random variable for each $1 \leq j \leq d$. 
By the standard tail bound for subgaussian random variables and the union bound, we have 
\begin{align*}
  \P_{\errorvec} \insquare{ \sup_{1 \leq j \leq d} [(\datamatrix^\top \datamatrix)^{-1/2} \datamatrix^\top \errorvec]_{j} \ge \sqrt{2 \log \frac{d}{\confLevel}}} \le \delta.
\end{align*}
As a result, we have with probability at least $1 - \confLevel$ 
\begin{align*}
  \norm{\empCov^{1/2}(\theta^\star - \thetaOLS)}_{p} = 
  \frac{1}{\sqrt{n}} \norm{(\datamatrix^\top \datamatrix)^{-1/2} \datamatrix^\top \errorvec}_{p} 
  \leq d^{1/p} \sqrt{ \frac{ 2 \log(d/\confLevel) }{n} },
\end{align*}
which implies $\theta^\star \in \thetalp$.

\paragraph{Step 2: proving $\sup_{\theta \in \Theta} \norm{\theta^\star - \theta} \le \beta$.}
By the triangle inequality, we have 
\begin{align*}
  \sup_{\theta \in \Theta} \norm{\theta^\star - \theta} \le 
  \sup_{\theta \in \Theta} \norm{\thetaOLS - \theta} + 
  \norm{\thetaOLS - \theta^\star} \leq \frac{\beta}{2} + \frac{\beta}{2} = \beta. 
\end{align*}
Here the second inequality relies on the definition of $\Theta$, and the consequence $\theta_\star \in \Theta$ of Step 1. 

\noindent Combining both steps finishes the proof of this lemma.  

\subsection{Proof of \Cref{cor:tabular}}\label{sec:proof-cor-tabular}
We state and prove a more general version of \Cref{cor:tabular} that provides a guarantee for each $\estPolicy_p$.

\begin{corollary}
In the tabular setting, the learning rule $\estPolicy_p$ with $\Theta$ given by \Cref{eq:thetalp} achieves the suboptimality with probability at least $1-\confLevel$:
\begin{align*}
\val(\optPolicy) - \val(\estPolicy_p) \lesssim (SA)^{1/p} \sqrt{ \frac{\log(SA/\confLevel)}{n} } \cdot \inparen{\sum_{s} \frac{\testDist^q(s)}{(n(\state,\optPolicy(\state))/n)^{q/2}}}^{1/q}.
\end{align*}
as long as $n \gtrsim \log(\numS/\confLevel) \cdot \inparen{\min_s\{\behDist(\state, \optPolicy(s))\} }^{-1}$.
\end{corollary}

\begin{proof}
We specialize \Cref{thm:lp-estimator-bound} to the tabular setting. In the tabular setting, we have $d=\numS \numA$. The empirical second moment matrix is a diagonal matrix with values $n(\state,\action)/n$. The vector $\E_{\state\sim \testDist}[\feature(\state, \optPolicy(\state))]$ takes the value in the $(\state,\action)$-th coordinate of $\rho(s) \indicator{a = \optPolicy(s)}$. Thus, we can derive the guarantee that with probability at least $1-\confLevel/2$,
  \begin{align*}
    \val(\optPolicy) - \val(\estPolicy_p) \lesssim (SA)^{1/p} \sqrt{ \frac{\log(SA/\confLevel)}{n} } \cdot \inparen{\sum_{s} \frac{\testDist^q(s)}{(n(\state,\optPolicy(\state))/n)^{q/2}}}^{1/q}.
  \end{align*}
  In order to convert this to a guarantee in terms of the behavior distribution, we use the following lemma.
  
  \begin{lemma}\label{lem:concentration-counts}
  If $n \ge 8 \log(SA/\confLevel) \cdot \inparen{\min_s\{\behDist(\state, \optPolicy(s))\} }^{-1} $, then with probability at least $1-\confLevel$:
  \begin{align*}
    n(\state,\optPolicy(\state)) \ge \frac{n \cdot \behDist(\state,\optPolicy(\state))}{2}, \text{ for all } \state.
  \end{align*}
  \end{lemma}
  \noindent Applying the lemma concludes the proof of the first part of the statement.   
\end{proof}

\noindent It remains to prove the lemma.

\paragraph{Proof of \Cref{lem:concentration-counts}.}
Fix any $s \in \stateS$. Using Chernoff bounds, one has
\begin{align*}
  \Pr \insquare{ n( \state, \optPolicy(\state) ) \le \frac{n}{2} \cdot \behDist( \state, \optPolicy(\state) } \le \exp(-n \behDist( \state, \optPolicy(\state) )/8).
\end{align*}
Thus we have by union bound, 
\begin{align*}
\Pr\insquare{\text{exists } s: n ( \state, \optPolicy(\state) ) \le \frac{n}{2} \cdot \behDist( \state, \optPolicy(\state) ) } 
\le \numS \cdot \exp(-n \behDistmin /8),
\end{align*}
where $\behDistmin \coloneqq \min_s\{\behDist(\state, \optPolicy(s))\}$.
By setting the RHS to $\confLevel$, we prove the result.\qed

\subsection{Tighter analysis for $\estPolicy_2$ in the tabular setting}\label{sec:tabular-l2}

In this section, we point out an issue with adapting the general result of \Cref{thm:lp-estimator-bound} to the tabular setting. This reveals a flaw with pessimism-validity, namely the requirement that $\theta^\star \in \Theta$.

Consider the single policy concentrability assumption of \Cref{def:concentrability}. It is straightforward to compute that for any $q \ge 1$, we have:
\begin{align*}
  \sum_{s} \frac{\testDist^q(s)}{\behDist^{q/2}(\state, \optPolicy(\state))} \le (C^\star)^{q/2} \cdot \sum_{s} \testDist^{q/2}(s) \le (C^\star)^{q/2} \cdot S^{1-q/2}.
\end{align*}
Therefore, in the tabular setting, the $\estPolicy_p$ recovers the guarantee of
\begin{align*}
  \val(\optPolicy) - \val(\estPolicy_p) \lesssim A^{1/p} \sqrt{ \frac{ SC^\star \log(SA/\confLevel) }{n} }.
\end{align*}

Therefore, under the single policy concentrability assumption, all $\estPolicy_p$ attain the minimax-optimal rate of $\tilde{O}(\sqrt{SC^\star/n})$ in the regime $C^\star \ge 2$ \cite{rashidinejad2021bridging}, up to the additional dependence on $A^{1/p}$. The polynomial dependence on $A$ can be removed by considering a smaller confidence set $\Theta$ for the tabular setting and directly analyzing the suboptimality. We illustrate how to do this for $p=2$; it is trivial to extend to all $p$.

\begin{proposition}\label{thm:l2-tight}
Fix any $\delta \in (0,1/2)$. In the tabular setting, the learning rule $\estPolicy_2$ configured with $\beta \coloneqq \sqrt{\frac{16 S \log(SA/\confLevel)}{n}}$ yields the suboptimality with probability at least $1-\delta$:
\begin{align*}
  \val(\optPolicy) - \val(\estPolicy_2) \lesssim \sqrt{ \frac{ SC^\star \log(SA/\confLevel) }{n} },
\end{align*}
as long as  $n \gtrsim \log(\numS/\confLevel) \cdot \inparen{\min_s\{\behDist(\state, \optPolicy(s))\} }^{-1}$.
\end{proposition}

\noindent Before we prove the theorem, we remark that this observation implies that the proof techniques relying on assuming $\theta^\star \in \Theta$ (i.e.,~the results in this paper, as well as extensions to general function classes \cite{xie2021bellman}) can be \emph{fundamentally loose}. Indeed one can check that we will not have $\theta^\star \in \Theta$ for the above learning rule with high probability. We leave developing a tighter analysis of pessimism via confidence sets that bypasses this assumption for general function approximation to future work.

\begin{proof}[Proof of \Cref{thm:l2-tight}.]
By \Cref{eq:thetalp-estimator}, we can write the $\estPolicy_2$ in the tabular setting as follows:
\begin{align*}
  \estPolicy_2 \coloneqq \argmax_{\policy \in \policyclass} \estVal(\policy), \quad \text{where } \estVal(\policy) = \inparen{ \sum_\state \testDist(\state) \empReward(\state, \policy(\state)) } - \frac{\beta}{2} \cdot \sqrt{\sum_\state \frac{\testDist^2(\state)}{n(\state, \policy(\state))/n}}.
\end{align*}
We will directly bound the suboptimality in view of \Cref{eq:subopt-decomposition}. We claim that
\begin{subequations}
	\begin{align}
		\val(\optPolicy) - \estVal(\optPolicy) &\lesssim \sqrt{\frac{SC^\star \log(SA/\delta)}{ n } } ; \label{eq:bound-term-1}\\
    \estVal(\estPolicy_2) - \val(\estPolicy_2) &\le 0. \label{eq:bound-term-2}
	\end{align}
\end{subequations}
We will prove the two inequalities conditioned on the following event:
\begin{align*}
  \calE \coloneqq \inbraces{ \text{for all } \state \in \stateS, \quad  r(\state, \optPolicy(\state)) - \empReward(\state, \optPolicy(\state)) \le \sqrt{\frac{4 \log (SA/\confLevel)}{n(\state, \optPolicy(\state))}} }.
\end{align*}
By Hoeffding's inequality and union bound, with probability at least $1-\confLevel$, $\calE$ holds.

\paragraph{Proof of \Cref{eq:bound-term-1}.} We compute the bound that 
\begin{align*}
  \val(\optPolicy) - \estVal(\optPolicy) &= \sum_\state \testDist(\state) \inparen{r(\state, \optPolicy(\state)) - \empReward(\state, \optPolicy(\state))} + \frac{\beta}{2} \cdot \sqrt{\sum_\state \frac{\testDist^2(\state)}{n(\state, \optPolicy(\state))/n}} \\
  &\stackrel{(i)}{\le} \sum_\state \testDist(\state) \sqrt{\frac{4 \log(SA/\delta)}{n(\state, \optPolicy(\state))}} + \frac{\beta}{2} \cdot \sqrt{\sum_\state \frac{\testDist^2(\state)}{n(\state, \optPolicy(\state))/n}} \\
  &\stackrel{(ii)}{\le} \sqrt{\sum_\state \testDist^2(\state) \cdot \frac{4 S\log(SA/\delta)}{n(\state, \optPolicy(\state)) }  } + \frac{\beta}{2} \cdot \sqrt{\sum_\state \frac{\testDist^2(\state)}{n(\state, \optPolicy(\state))/n}}\\
  &\stackrel{(iii)}{=} \sqrt{\sum_\state \testDist^2(\state) \cdot \frac{16 S\log(SA/\delta)}{n(\state, \optPolicy(\state)) }  }.
\end{align*}
Inequality (i) follows by Hoeffding's inequality, (ii) is due to Cauchy-Schwarz, and (iii) is the definition of $\beta$. After applying Chernoff bounds and applying the assumption of single policy concentrability, we see that with probability at least $1-\delta$,
\begin{align*}
  \val(\optPolicy) - \estVal(\optPolicy) &\lesssim \sqrt{\frac{ SC^\star \log(SA/\delta)}{ n } }.
\end{align*}

\paragraph{Proof of \Cref{eq:bound-term-2}.} The proof for this inequality follows a similar outline. Fix any policy $\policy$. We can compute that with probability at least $1-\delta$:
\begin{align*}
  \estVal(\policy) - \val(\policy) &= \sum_\state \testDist(\state) \inparen{\empReward(\state, \policy(\state))  - r(\state, \policy(\state)) } - \frac{\beta}{2} \cdot \sqrt{\sum_\state \frac{\testDist^2(\state)}{n(\state, \policy(\state))/n}} \\
  &\le \sqrt{\sum_\state \testDist^2(\state) \cdot \frac{4 S\log(SA/\delta)}{n(\state, \policy(\state)) }  } - \frac{\beta}{2} \cdot \sqrt{\sum_\state \frac{\testDist^2(\state)}{n(\state, \policy(\state))/n}} \le 0,
\end{align*}
again using Hoeffding's inequality, Cauchy-Schwarz, and the definition of $\beta$.

\medskip
\noindent By combining the two inequalities we prove the proposition. 
\end{proof}

\section{Proof of \Cref{thm:lp-lower-bound}}\label{sec:improved-lower-bound}

The proof of the lower bound relies on a reduction to a bound on the minimax regret of the offline multi-armed bandit problem \cite{xiao2021optimality}. We will utilize the following lemma, which lower bounds the minimax regret for tabular contextual bandit instances.

\begin{lemma}\label{lem:tabular-lb}
Let $\numS$ be arbitrary and $\numA \ge 2$. Fix any test distribution $\rho \in \Delta(\stateS)$ and counts ${\bf n} = \{n(\state, \action)\}_{\state \in \stateS, \action \in \actionS}$. Define the set of contextual bandit instances:
\begin{align*}
    \mathsf{CB}_{\rho, {\bf n}} \coloneqq \inbraces{ R: R(\state,\action) \text{ is } 1\text{-subgaussian for all } (\state,\action) }.
\end{align*}
Then there exists a universal constant $c > 0$ such that
\begin{align}\label{eq:tab-lower-bound}
    \inf_{\estPolicy} \sup_{ \CBinstance \in \mathsf{CB}_{\rho, {\bf n}} } \E[ \optVal_{ \CBinstance } - \val_{ \CBinstance }( \estPolicy ) ] 
    \ge \const \cdot \sum_\state \inparen{ \testDist(\state) \max_{\action \in \{ \actioni{1}, \actioni{2} \} } \frac{1}{\sqrt{n(s,a)}} }.
\end{align}
\end{lemma}

\begin{proof}
Theorem 1 of \citet{xiao2021optimality} proves a lower bound on the Bayes suboptimality for any multi-armed bandit instance. Specifically, they show that for any $A$-armed bandit problem, one can define a collection of instances $\calB$; for any sequence of counts ${\bf \n} = \{n(a)\}_{a \in \actionS}$, the Bayes suboptimality is at least:
\begin{align}\label{eq:bayes-lb}
    \inf_{\estPolicy} \E_{\CBinstance \sim \mathrm{Unif}(\calB)}[\reward^\star_\CBinstance - \reward_\CBinstance(\estPolicy)] \gtrsim \max_{a \in \actionS} \frac{1}{\sqrt{n(a)}}.
\end{align}
In order to prove a guarantee for the \emph{tabular contextual bandit problem}, we will tensorize this result. Intuitively, we can treat the estimation of the policy for each state as a separate multi-armed bandit policy estimation problem. Formally, the Bayes suboptimality bound of \Cref{eq:bayes-lb} allows us to do this.
Note that the suboptimality for any policy $\policy$ can be written as
\begin{align*}
    \val_\CBinstance^\star - \val_\CBinstance(\policy)  = \sum_{\state} \testDist(\state) \cdot \inparen{ r_\CBinstance(\state, \optPolicy) - r_\CBinstance(\state, \policy) }.
\end{align*}
Therefore, we can lower bound the minimax suboptimality of the contextual bandit problem as the sum of the Bayes suboptimalities for the $S$ individual bandit problems:
\begin{align*}
    \inf_{\estPolicy} \sup_{ \CBinstance \in \mathsf{CB}_{\rho, {\bf n}} } \E[ \optVal_{ \CBinstance } - \val_{ \CBinstance }( \estPolicy ) ]
    \ge \sum_\state \testDist(\state) \cdot \inparen{ \inf_{\estPolicy_\state} \ \E_{\calQ_s \sim \mathrm{Unif}(\calB_\state)} \insquare{ \reward^\star_{\CBinstance_s} - \reward_{\CBinstance_s}(\estPolicy_s)} }.
\end{align*}
Here, $\estPolicy_s$ is taken to be any estimation procedure that looks only at the subset of the dataset over state $\state$; the collection $\calB_s$ is the set of bandit instances defined by \citet{xiao2021optimality} (for the counts $\{n(\state,\action)\}_{\action\in\actionS}$). The inequality follows by the lower bound for the minimax risk of a tensor product (see, e.g.,~\cite{wunotes}, Theorem 3.1). The proof concludes by applying the bounds on the Bayes risk for each state in \eqref{eq:bayes-lb} to the previous display.
\end{proof}

\noindent With \Cref{lem:tabular-lb} in hand, we now prove the theorem.

\medskip
\paragraph{Proof of \Cref{thm:lp-lower-bound}.} The family of hard instances we construct are tabular contextual bandits. We begin by describing the state and action spaces. Define $S = d/2$ (in the case where $d$ is odd, we set $S = \floor{d/2}$ and add a ``dummy'' coordinate to the tabular feature mapping, which does not affect the rest of the proof). We set $\stateS = \{1, 2, \dots, \numS\}$ and $\actionS = \{ \actioni{1}, \actioni{2} \}$. Thus, we use the tabular feature mapping, \ie $\feature(\state,\action) = e_{\state,\action}$.

In view of \Cref{lem:tabular-lb}, our task is as follows. For each tuple $(\Lambda, q, n)$ we need to supply a test distribution $\rho$ and counts ${ \bf n }$ which satisfies two properties:
\version{\vspace{-\topsep}}{}
\version{\begin{enumerate}[leftmargin=0.5cm]}{\begin{enumerate}}
    \item $\mathsf{CB}_{\rho, {\bf n}} \subseteq \mathsf{CB}_q(\Lambda)$. In other words, we need $\rho$ and ${\bf n}$ to satisfy the inequality
    \begin{align*}
        \lVert \empCov^{-1/2} \E_{\state \sim \testDist} [\phi(\state, \optPolicy_v(\state))] \rVert_q = \inparen{ \sum_\state \testDist^q(\state) \max_{\action \in \{ \actioni{1}, \actioni{2} \} } \frac{1}{(n(\state,\action)/n)^{q/2}} }^{1/q} \leq \Lambda.
    \end{align*}
    \item The RHS of \Cref{eq:tab-lower-bound} is sufficiently large:
    \begin{align*}
        \sum_\state \testDist(\state) \max_{\action \in \{ \actioni{1}, \actioni{2} \} } \frac{1}{\sqrt{n(s,a)/n}}  \gtrsim d^{1/p} \Lambda.
    \end{align*}
\end{enumerate}

\paragraph{Hard instance construction.}
Fix any value of $\Lambda \ge \sqrt{8} d^{1/2-1/p}$. We will utilize the construction from \citet{rashidinejad2021bridging}. Let us set:
\begin{align*}
    \testDist = \mathrm{Unif}(\{1, \dots, S\}), \quad \text{and} \quad 
    \begin{array}{l}
        n(\state, \actioni{1}) = \floor{n/S \cdot (1/\Gamma)}, \\
        n(\state, \actioni{2}) = \floor{n/S \cdot  (1-1/\Gamma)}.
    \end{array}
\end{align*}
The parameter $\Gamma$ is set to be $\Gamma = S^{2/p-1}\Lambda^2/2$; without loss of generality we will require $\Gamma \ge 2$ so that $n(\state, \actioni{1}) \le n(\state, \actioni{2})$. Note that as long as $n \ge S^{2/p}\Lambda^2 = 2S\Gamma$ (which holds whenever $n \ge d^{2/p}\Lambda^2$), we have $n(\state, \actioni{1}) \ge 1/2 \cdot n/S \cdot (1/\Gamma)$ and $n(\state, \actioni{2}) \ge 1/2 \cdot n/S \cdot (1-1/\Gamma)$. It is easy to calculate the bounds that:
\begin{align*}
    \sum_\state \testDist^q(\state) \max_{\action \in \{ \actioni{1}, \actioni{2} \} } \frac{1}{(n(\state,\action)/n)^{q/2}}  &\le \Lambda^q,\\
    \sum_\state \testDist(\state) \max_{\action \in \{ \actioni{1}, \actioni{2} \} } \frac{1}{\sqrt{n(\state,\action)/n}} &\ge S^{1/p} \Lambda/\sqrt{2} \gtrsim d^{1/p} \Lambda,
\end{align*}
thus the conditions are satisfied. Lastly, we note that in order to ensure $\Gamma \ge 2$, we need $\Lambda \ge 2 S^{1/2-1/p}$; using the fact that $S = d/2$, a sufficient condition is $\Lambda \ge \sqrt{8} d^{1/2-1/p}$. By reduction to the contextual bandit lower bound, this proves the result.

\paragraph{Improving the range of $\Lambda$ for $q=1$.} We show how we can prove the theorem for an extended range of $\Lambda$ when $q=1$. Fix any value of $\Lambda \ge 2$. We will slightly tweak the previous construction to be $T$-sparse for some $T \in \{1, \dots, S\}$. Let us set:
\begin{align*}
    \testDist = \mathrm{Unif}(\{1, \dots, T\}), \quad \text{and} \quad 
    \begin{array}{l}
        n(\state, \actioni{1}) = \floor{n/T \cdot \Gamma}, \\
        n(\state, \actioni{2}) = \floor{n/T \cdot (1-1/\Gamma)}.
    \end{array}
\end{align*}
We will pick parameters $T$ and $\Gamma$ such that $\Gamma \ge 2$ and $T \coloneqq \floor{\Lambda^2/(2\Gamma)}$. As long as $\Lambda \ge 2$, such a choice is indeed valid because when $\Gamma = 2$ (the smallest possible choice for $\Gamma$), we have $\Lambda^2/\Gamma \ge 1 \Rightarrow T \ge 1$; also, we can always pick $\Gamma$ sufficiently large so that $T \le S$. In addition, we note that under the sample complexity requirement $n \ge \Lambda^2 \ge 2T\Gamma$, we have $n(\state, \actioni{1}) \ge 1/2 \cdot n/T \cdot 1/\Gamma$ and $n(\state, \actioni{2}) \ge 1/2 \cdot n/T \cdot (1-1/\Gamma)$.

Now it remains to verify the conditions. We calculate that
\begin{align*}
     \Lambda \lesssim \sqrt{\Gamma T} \le \sum_\state \testDist(\state) \max_{\action \in \{ \actioni{1}, \actioni{2} \} } \frac{1}{\sqrt{n(s,a)/n}} \le \Lambda.
\end{align*}
Thus, in the case where $p=\infty$, $q=1$, we have proven that the instance satisfies both conditions.\qed

\subsection{Limitations of the lower bound}\label{sec:lb-limitations}

We discuss in more detail the technical challenges with strengthening the lower bound by proving that the condition holds over a larger range of $\Lambda$. Since we argue that $\Lambda$ roughly corresponds to the difficulty of the offline policy learning problem, it would be desirable to show that the lower bound still holds even when $\Lambda$ is small. We are able to do this for the $p=\infty, q=1$ case, but extending this construction to finite $p$ seems challenging.

\paragraph{Fundamental limitation of tabular design.} First, we observe that there is already a fundamental limitation in the reduction to contextual bandits; namely, we cannot hope to prove \Cref{thm:lp-lower-bound} for any value of complexity parameter $\Lambda \le O(1/d^{1/p})$. Observe that for any tabular contextual bandit, the key $\ell_1$ quantity takes the form:
\begin{align*}
    \sum_\state \testDist(\state) \max_{\action \in \{ \actioni{1}, \actioni{2} \} } \frac{1}{\sqrt{n(s,a)/n}}.
\end{align*}
Since $\testDist \in \Delta(\stateS)$ and $n(s,a) \le n$ for all $(s,a)$ pairs, this quantity must be $\Omega(1)$. Using the identity $\norm{v}_1 \le d^{1/p} \norm{v}_q$, we see that it is not even possible to construct tabular contextual bandit instances unless $\Lambda \ge \Omega(d^{-1/p})$. Thus we must look beyond tabular lower bound constructions.

\paragraph{Extending $\Lambda$ range for $q \in (1,2]$.} We provide an example of how extend the lower bound to a larger range of $\Lambda$ when $q \in (1,2]$ using a more complicated tabular construction. The catch is that the sample complexity increases from $n \ge d^{2/p} \Lambda^2$ to $n \ge d^{2/q} \Lambda^2$. We conjecture that this is an artifact of our analysis, and that it is possible to prove the lower bound under the requirement $n \ge d^{2/p} \Lambda^2$ (which would allow us to show a lower bound that takes all values $d^{1/p} \Lambda /\sqrt{n} \in (0, O(1)]$).

\begin{proposition}
    For every $d \geq 2$, there exists a feature mapping $\feature$ such that the following lower bound holds. Fix any $q \in (1,2]$. As long as $\Lambda \ge \sqrt{12}$ and $n \ge d^{2/q} \Lambda^2$, we have
    \begin{align*}
        \inf_{\estPolicy} \sup_{\CBinstance \in \mathsf{CB}_q(\Lambda)} \E
        [\optVal_{ \CBinstance } - \val_{ \CBinstance }( \estPolicy ) ] \ge c \cdot  d^{1/p} \sqrt{ \frac{1}{n} } \cdot \Lambda,
    \end{align*}
    where $c > 0$ is some universal constant.
\end{proposition}

\begin{proof}
We will modify the construction. We redefine the state space to be $\{0, 1, \dots, S\}$ (adding an extra state). This only changes everything by a constant at most but makes the definitions simpler. State 0 is a special state with:
\begin{align*}
    \testDist(0) = \rho_0, \quad \text{and} \quad 
    \begin{array}{l}
        n(0,\actioni{1}) = \floor{n\rho_0 / \Gamma_0}, \\
        n(0,\actioni{2}) = \floor{n\rho_0\inparen{1-1/\Gamma_0}}.
    \end{array}
\end{align*}
The other states $\state \ge 1$ are set as follows:
\begin{align*}
    \testDist(\state) = (1-\rho_0)/S, \quad \text{and} \quad 
    \begin{array}{l}
        n(s,\actioni{1}) = \floor{n\frac{1-\rho_0}{S} \cdot \frac{1}{\Gamma_1}}, \\
        n(s,\actioni{2}) = \floor{n\frac{1-\rho_0}{S}\inparen{1-\frac{1}{\Gamma_1}}}.
    \end{array}
\end{align*}
Again, as long as $n \ge 2 \max\inparen{\Gamma_0/ \rho_0, S \Gamma_1 / (1-\rho_0) }$, we have
\begin{align*}
    &n(0,\actioni{1}) \ge 1/2 \cdot (n\rho_0 / \Gamma_0),
    &n(0,\actioni{2}) \ge 1/2 \cdot n\rho_0\inparen{1-1/\Gamma_0}, \\
    &n(s,\actioni{1}) \ge 1/2 \cdot n\frac{1-\rho_0}{S} \cdot \frac{1}{\Gamma_1}, 
    &n(s,\actioni{2}) \ge n\frac{1-\rho_0}{S}\inparen{1-\frac{1}{\Gamma_1}}.
\end{align*}
We will now pick the values of $\rho_0 \in [0,1]$ and $\Gamma_0, \Gamma_1 \ge 2$ in order to satisfy the properties enumerated in the proof of \Cref{thm:lp-lower-bound}.

First, we compute expressions for the bounds. We see that
\begin{align*}
    \sum_\state \testDist(\state) \max_{\action \in \{ \actioni{1}, \actioni{2} \} } \frac{1}{\sqrt{n(\state,\action)/n}} &\ge \sqrt{\rho_0 \Gamma_0} +  \sqrt{S(1-\rho_0)\Gamma_1}, \\
    \inparen{ \sum_\state \testDist^q(\state) \max_{\action \in \{ \actioni{1}, \actioni{2} \} } \frac{1}{(n(\state,\action)/n)^{q/2}} }^{1/q} &\le \inparen{(2\rho_0 \Gamma_0)^{q/2} + S^{1-q/2} (2(1-\rho_0)\Gamma_1)^{q/2} }^{1/q},
\end{align*}
We pick
\begin{align*}
    \rho_0 = 1 - S^{1-2/q}/4, \quad \Gamma_0 = \Lambda^2/(8\rho_0), \quad \Gamma_1 = \Lambda^2/2.
\end{align*}
One can calculate that
\begin{align*}
     &\sum_\state \testDist^q(\state) \max_{\action \in \{ \actioni{1}, \actioni{2} \} } \frac{1}{(n(\state,\action)/n)^{q/2}}  \le 2(\Lambda/2)^q \le \Lambda^q,\\
     &\sum_\state \testDist(\state) \max_{\action \in \{ \actioni{1}, \actioni{2} \} } \frac{1}{n(\state,\action)/n} \ge \Lambda/2 + S^{1/p}\Lambda/2 \gtrsim S^{1/p}\Lambda.
\end{align*}
When $q\in (1,2]$, we have the bound that $\rho_0 \in (3/4, 1]$. In order for $\Gamma_0$ and $\Gamma_1$ to satisfy the requirement that $\Gamma_0, \Gamma_1 \ge 2$, we require $\Lambda \ge \sqrt{12}$. In terms of sample complexity, we require
\begin{align*}
    n \ge \max\inbraces{\frac{4}{9} \Lambda^2,  2 S^{2/q} \Lambda^2},
\end{align*}
which is satisfied whenever $n \ge d^{2/q} \Lambda^2$.
\end{proof}

\section{Separations between pessimistic learning rules}\label{sec:separations}

We begin by stating a formal version of \Cref{thm:separation-l2-linfty-new}.
\begin{theorem}\label{thm:separation-l2-linfty-formal}
Fix any $p \ge 1$. Fix any dimension $d\ge 20$, sample size $n \ge 9d^3$, and $\xi(d,n)$, where $\xi: \mathbb{Z}_{+}\times \mathbb{Z}_{+} \to \R_{+}$ is a functional that returns a positive real that satisfies $\xi \ge K_\xi $ for an absolute numerical constant $K_\xi > 0$. Then there exists a contextual bandit instance $\CBinstance \in \mathsf{CB}_1(\sqrt{8d})$ such that:
\version{\vspace{-\topsep}}{}
\version{\begin{enumerate}[leftmargin=0.5cm]}{\begin{enumerate}}
  \item With probability at least $1/4$, the learning rule $\estPolicy_p$ configured with $\beta = \xi \cdot d^{1/p} / \sqrt{n}$ has suboptimality at least
  \begin{align*}
    \val(\optPolicy) - \val( \estPolicy_p ) \ge \frac{K_\xi}{\sqrt{8}} \cdot \frac{d^{1/p+1/2}}{\sqrt{n}}.
  \end{align*}
  This implies a lower bound on the expected suboptimality of $\Omega(K_\xi \cdot d^{1/p+1/2}/\sqrt{n})$.
  \item The learning rule $\estPolicy_\infty$ configured with $\beta = \sqrt{\frac{8\log (K_\xi d^{5/2}/\sqrt{8})}{n}}$ has expected suboptimality at most
  \begin{align*}
    \E_\dataset \insquare{ \val(\optPolicy) - \val( \estPolicy_\infty ) } \le c \cdot \sqrt{\frac{d \log(K_\xi d)}{n}},
  \end{align*}
  where $c > 0$ is an absolute numerical constant.
\end{enumerate}

\end{theorem}

\medskip 

\noindent The rest of the section is devoted to the proof of \Cref{thm:separation-l2-linfty-formal}.

\paragraph{Construction of the instance.} The instance we construct is a tabular contextual bandit instance with the canonical basis as the feature mapping. We set $S = d/2$\footnote{Again, if $d$ is odd then we let $S = \floor{\frac{d}{2}}$ which does not affect the rest of the proof.} and let the state space be $\stateS = \{1, 2, \ldots, \numS\}$. Also we choose the action space be $\actionS = \{\actioni{1}, \actioni{2}\}$. For simplicity, we assume that the sample size $n$ is an integer multiple of $S^3$.\footnote{Otherwise we can pad the samples with at most $S^3$ dummy state-action pairs.} For each state $\state$, we define $\testDist( \state ) = \frac{1}{\numS}$. Now we will define the empirical counts for each $(\state,\action)$ pair:
\begin{align*} 
  n(\state, \actioni{1}) = \begin{cases}
    \frac{n}{9\numS^3}, &\text{if} \ \state = 1,\\
    \frac{n}{\numS}, &\text{otherwise},
  \end{cases}
  \qquad \text{and} \qquad
  n(\state, \actioni{2}) = \begin{cases}
    \frac{n}{\numS} - \frac{n}{9\numS^3}, &\text{if} \ \state = 1,\\
    0, &\text{otherwise}.
  \end{cases}
\end{align*}
Let $\gamma > 0$ be a \emph{gap parameter} that we will specify later. We set the reward distributions as follows:
\begin{align*}
  \rewardDist(1,a) = \begin{cases}
    \gauss{\gamma}{1}, & \action = \actioni{1},\\
    \gauss{0}{1}, & \action = \actioni{2},
  \end{cases} 
  \qquad \text{and} \qquad
   \rewardDist( \state, \action ) = \begin{cases}
    \frac{1}{\sqrt{\n}}, & \action = \actioni{1},\\
    0, & \action = \actioni{2},
  \end{cases}
  \quad \text{ for } \state \in \{2, 3, \dots, \numS \}.
\end{align*}
It is seen that for this instance, $\optPolicy( \state ) = \actioni{1}$ for all 
$\state \in \stateS$. The difficulty of offline learning lies solely in 
selecting the optimal action for $\state=1$.

Lastly, we verify that the instance indeed lies in $\mathsf{CB}_1(\Lambda)$ for $\Lambda = \sqrt{8d}$. Algebraic manipulation yields
\begin{align}\label{eq:l1-calculation}
  \lVert \empCov^{-1/2} \E_{\state \sim \testDist} [\feature(\state, \optPolicy(\state))] \rVert_1 = \sum_\state \frac{\testDist(s)}{\sqrt{n(\state,\optPolicy(s))/n}} = \frac{1}{\numS} \inparen{3S^{3/2} + (S-1)S^{1/2}} = 4S^{1/2} - S^{-1/2}.
\end{align}
Plugging in $S = d/2$ proves that $\lVert \empCov^{-1/2} \E_{\state \sim \testDist} [\feature(\state, \optPolicy_v(\state))] \rVert_1 \le \Lambda$, \ie the CB instance is in $\mathsf{CB}_1(\Lambda)$ for $\Lambda = \sqrt{8d}$.

\paragraph{Lower bounding the performance of $\estPolicy_p$.} Now we will prove that $\estPolicy_p$ achieves expected suboptimality which scales with $d^{1/p+1/2}$. For the tabular setting, $\estPolicy_p$ can be rewritten as (cf.~\Cref{eq:thetalp-estimator}):
\begin{align*}
  \estPolicy_p = \argmax_{\policy: \stateS\to \actionS} \inbraces{ \sum_\state \testDist(\state) \empReward(\state, \policy(\state)) - \beta \inparen{\sum_\state \frac{\testDist^p(s)}{(n(\state,\policy(s))/n)^{p/2}} }^{1/p} }, \quad \text{where }\beta \coloneqq \xi \frac{(2S)^{1/p}}{\sqrt{n}}. 
\end{align*}
By the definition of $\empReward(\state, \action)$ and $\n(\state, \action)$, 
we know that $\estPolicy_p(\state) = \optPolicy( \state ) = \actioni{1}$ 
for $\state \ge 2$. Therefore, we have that the suboptimality $\val ( \optPolicy ) - \val(\estPolicy_p) \ge \gamma / S$ under the event $\inbraces{ \estPolicy_p(1) = \actioni{2} }$.
Therefore our goal boils down to lower bounding the probability 
$\P_{\dataset} \insquare{ \estPolicy_p(1) = \actioni{2} }$. We can rewrite the learning rule on $\state=1$ as follows:
\begin{align*}
  \estPolicy_p(1) = \argmax_{ \action \in \{\actioni{1}, \actioni{2}\} } 
  \quad  \empReward(1, \action) - \xi (2S)^{1/p} 
  \Big( \tfrac{1}{ \n(1, \action)^{p/2} } + \sum_{\state \ge 2} 
  \tfrac{1}{ \n(\state, \actioni{1})^{p/2} } \Big)^{1/p} \; \eqqcolon \;
   \argmax_{ \action \in \{\actioni{1}, \actioni{2}\} } \quad 
   { \tilde{\reward}(1, \action)},
\end{align*}
where we use the observation that $\estPolicy_p(\state) = \optPolicy( \state ) = \actioni{1}$ for $\state \ge 2$.
As a result, we have the identity
\begin{align*}
\P_{\dataset} \insquare{ \estPolicy_p(1) = \actioni{2} } = 
\P_{\dataset} \insquare{ \tilde{\reward}( 1, \actioni{2} ) > 
\tilde{\reward}( 1, \actioni{1} )}.
\end{align*}
The lower bound of the right hand side is provided in the following lemma, 
whose proof is deferred to the end of this section.  

\begin{lemma}\label{lem:lower-bound-prob-v2}
  Setting the gap parameter $\gamma = K_\xi  \frac{ S^{1/p+3/2} }{ \sqrt{\n } }$, one has $\P_{\dataset} \insquare{ \tilde{\reward}( 1, \actioni{2} ) > 
  \tilde{\reward}( 1, \actioni{1} )} \geq \tfrac{1}{4}$.
\end{lemma}

\noindent Putting together the previous claims yields the performance guarantee stated in part (1).

\paragraph{Upper bounding the performance of $\estPolicy_\infty$.} By \Cref{thm:lp-estimator-bound}, under the choice $\beta \coloneqq \sqrt{\frac{8\log (2S/\confLevel)}{n}}$, with probabilty at least $1-\confLevel$, the learning rule $\estPolicy_\infty$ achieves the guarantee
\begin{align*}
  \val(\optPolicy) - \val( \estPolicy_\infty ) \le \sqrt{\frac{8 \log(d/\confLevel)}{n}} \cdot \lVert \empCov^{-1/2} \E_{\state \sim \testDist} [\feature(\state, \optPolicy_v(\state))] \rVert_1 \le \sqrt{\frac{32S \log(2S/\confLevel)}{n}},
\end{align*}
where the inequality uses \Cref{eq:l1-calculation}. Thus the expected suboptimality is at most
\begin{align*}
  \E_\dataset\insquare{ \val(\optPolicy) - \val( \estPolicy_\infty ) } \le (1-\confLevel) \cdot \sqrt{\frac{32S \log(2S/\confLevel)}{n}} + \confLevel \cdot \inparen{ \frac{K_\xi S^{3/2}}{\sqrt{n}} + \frac{1}{\sqrt{n}}}.
\end{align*}
The second term in the expression comes from the fact that on this instance, the worst suboptimality any policy can incur is at most:
\begin{align*}
  \frac{1}{S} \inparen{ \gamma + (S-1) \cdot \frac{1}{\sqrt{n}} } \le \frac{K_\xi S^{1/p+1/2}}{\sqrt{n}} + \frac{1}{\sqrt{n}}.
\end{align*}
Now we pick $\confLevel = 1/(K_\xi S^{3/2})$, which shows that when $\beta = \sqrt{\frac{8\log (2S^{5/2} K_\xi)}{n}}$, we achieve the guarantee in expectation
\begin{align*}
  \E_\dataset\insquare{ \val(\optPolicy) - \val( \estPolicy_\infty ) } \le c \cdot \sqrt{\frac{d \log(K_\xi d)}{n}},
\end{align*}
for some absolute numerical constant $c > 0$.

\subsection{Proof of \Cref{lem:lower-bound-prob-v2}}\label{sec:lower-bound-prob}

We calculate that
\begin{align*}
  \P_{\dataset} \insquare{ \tilde{\reward}( 1, \actioni{2} ) > 
  \tilde{\reward}( 1, \actioni{1} )}
  = \: &\P_{\dataset} \Bigg[  \empReward(1, \actioni{2} ) - \xi (2S)^{1/p}  
  \Big( \tfrac{1}{ \n(1, \actioni{2} )^{p/2} } + \sum_{\state \ge 2} 
  \tfrac{1}{ \n(\state, \actioni{1})^{p/2} } \Big)^{1/p}  \\
  &\quad\quad  > \empReward(1, \actioni{1} ) - \xi (2S)^{1/p}  
  \Big( \tfrac{1}{ \n(1, \actioni{1} )^{p/2} } + \sum_{\state \ge 2} 
  \tfrac{1}{ \n(\state, \actioni{1})^{p/2} } \Big)^{1/p} \Bigg] \\
  = \: &\P_{\dataset} \Bigg[  \empReward(1, \actioni{2} ) - \xi \frac{(2S)^{1/p}}{\sqrt{n}}  
  \Big( \inparen{\tfrac{1}{S} - \tfrac{1}{9S^3}}^{-p/2} + (S-1) S^{p/2} \Big)^{1/p}  \\
  &\quad\quad > \empReward(1, \actioni{1} ) - \xi \frac{(2S)^{1/p}}{\sqrt{n}}   
  \Big( (9S^3)^{p/2} + (S-1)S^{p/2} \Big)^{1/p} \Bigg].
\end{align*}
Note that the empirical reward $\empReward(1, \actioni{1} )$ is distributed as a mean-$\gamma$ Gaussian, so we can lower bound the previous display as
\begin{align*}
  &\P_{\dataset} \insquare{ \tilde{\reward}( 1, \actioni{2} ) > 
  \tilde{\reward}( 1, \actioni{1} )} \\
  &\quad \geq \ \frac{1}{2} \cdot \P_{\dataset} \Bigg[  \empReward(1, \actioni{2} ) > \gamma \\
  & \quad\quad + \xi \frac{(2S)^{1/p}}{\sqrt{n}} \inbraces{\Big( \inparen{\tfrac{1}{S} - \tfrac{1}{9S^3}}^{-p/2} + (S-1) S^{p/2} \Big)^{1/p} -\Big( (9S^3)^{p/2} + (S-1)S^{p/2} \Big)^{1/p}  } \Bigg] 
  .
\end{align*}
Now observe that under the choice of gap parameter $\gamma = K_\xi  \frac{ S^{1/p+3/2} }{\sqrt{\n} } $, for sufficiently large $S$ (say $S\ge 10$) we have:
\begin{align*}
  &\gamma + \xi \frac{(2S)^{1/p}}{\sqrt{n}} \inbraces{\Big( \inparen{\tfrac{1}{S} - \tfrac{1}{9S^3}}^{-p/2} + (S-1) S^{p/2} \Big)^{1/p} -\Big( (9S^3)^{p/2} + (S-1)S^{p/2} \Big)^{1/p}  } \\
  &\quad  < \gamma + K_\xi \frac{(2S)^{1/p}}{\sqrt{n}} \inbraces{\Big( \inparen{\tfrac{1}{S} - \tfrac{1}{9S^3}}^{-p/2} + (S-1) S^{p/2} \Big)^{1/p} -\Big( (9S^3)^{p/2} + (S-1)S^{p/2} \Big)^{1/p}  } \\
  &\quad < 0.
\end{align*}
Using the fact that $\empReward(1, \actioni{2} )$ is distributed as a mean-zero Gaussian, we can further lower bound the probability as $P_{\dataset} \insquare{ \tilde{\reward}( 1, \actioni{2} ) > 
\tilde{\reward}( 1, \actioni{1} )} \geq \tfrac{1}{4}$, thus proving the lemma.

\section{Plug-in estimation performance}\label{sec:plugin}
In this section, we discuss the performance of the simple plug-in rule:
\begin{align*}
  \plugin(\state) \coloneqq \argmax_{\action\in \actionS} \phi(s,a)^\top \thetaOLS.
\end{align*}
This is a natural baseline to study. It is an instantiation of the so-called \emph{certainty equivalence principle} from control theory, which is shown to be optimal for some online control problems \citep[see, e.g.,][]{mania2019certainty, simchowitz2020naive}.

\subsection{Performance upper bound}
\begin{proposition}[Folklore; see, e.g.,~\cite{duan2021risk}, pg.~50]\label{prop:plugin-bound}
Assume that $\norm{\feature(\state,\action)}_2 \le B$. Then with probability at least $1-\confLevel$, $\plugin$ achieves a suboptimality 
\begin{align*}
  \val(\optPolicy) - \val(\plugin) \le \sqrt{\frac{8 d \log(d/\confLevel)}{n}} \cdot B \cdot \lambdamin(\empCov)^{-1/2}.
\end{align*}
\end{proposition}

\begin{proof}
Let us denote $\estVal(\pi) \coloneqq \phi(s,a)^\top \thetaOLS$. We use \Cref{eq:subopt-decomposition} and note that 
\begin{align*}
  \val(\optPolicy) - \estVal(\optPolicy) &= \Ex{\state \sim \testDist }{ \feature(\state,\optPolicy(\state) )^\top (\theta^\star - \thetaOLS)}, \\
  \estVal(\plugin) - \val(\plugin) &= \Ex{\state \sim \testDist}{\feature(\state,\plugin(\state))^\top (\thetaOLS - \theta^\star)}.
\end{align*} 
Note that for any $(s,a)$:
\begin{align*}
  \feature(\state,\action)^\top (\theta^\star - \thetaOLS) \le B \norm{\theta^\star - \thetaOLS}_2 \le \norm{\empCov^{1/2} (\theta^\star - \thetaOLS)}_2 \cdot B\cdot \lambdamin(\empCov)^{-1/2}. 
\end{align*}
Lastly, the $\ell_2$ norm term is bounded as a consequence of \Cref{lemma:valid-lp}, which shows that:
\begin{align*}
  \norm{\empCov^{1/2}(\theta^\star - \thetaOLS)}_2 \le \sqrt{\frac{2 d \log(d/\confLevel)}{n}}.
\end{align*}
(The $\log d$ factor can be removed with a more involved argument.) This proves the result.
\end{proof}

Note that \Cref{prop:plugin-bound} gives a dependence on $B \cdot \lambdamin(\empCov)^{-1/2}$, which is always worse than the guarantee for $\estPolicy_2$ learning rule (where the relevant ``complexity measure'' is $\complexity_2$). The plug-in rule requires the covariates $\{(\feature(\state_i, \action_i))\}_{i=1}^n$ to cover all directions well in order to obtain low suboptimality, even if the optimal policy is well-covered by the dataset. The benefit of pessimism is that when the features cover the optimal policy well, one can achieve better performance.

\subsection{Separation between plug-in estimation and pessimism}
Now we state a separation which shows when pessimism is prefered over the plug-in rule. This result is included for completeness; a version of it appears in prior work, see, e.g.,~Proposition 1 from the paper \cite{rashidinejad2021bridging}. It is stated for the random design setting (where we assume the covariates $(\state,\action)$ are drawn i.i.d.~from a behavior distribution $\behDist$). This implies a fixed design result (in line with the other results of this paper).

\begin{proposition}\label{prop:separation-erm-gp-fixed}
Fix any $\numA \ge 8$. There exists a multi-armed bandit instance with 
$\numA$ actions such that 
\version{\vspace{-\topsep}}{}
\version{\begin{enumerate}[leftmargin=0.5cm]}{\begin{enumerate}}
  \item For any $\n \le 2^\numA$, the plug-in rule obeys 
  \begin{align*}
  \E_\dataset \insquare{ \val(\optPolicy) - \val(\estPolicy_\mathsf{plug})} \ge \const_1.
  \end{align*}
  \item As long as $\n \ge 200$, the learning rule $\estPolicy_\infty$ achieves the guarantee
  \begin{align*}
    \E_\dataset \insquare{ \val(\optPolicy) - \val( \estPolicy_\infty ) } \le 
  \const_2 \sqrt{ \frac{  \log (\numA \n) }{ \n } }.
  \end{align*}
\end{enumerate}
Here $\const_1, \const_2 > 0$ are two universal constants. 
\end{proposition}

\begin{proof}
\textbf{Proof of part (1). } We specify the behavior distribution as follows
\begin{align*}
  \behDist(\action_k) = 2^{-k}\quad  \text{ for } 1 \leq k \leq \numA-1, 
  \quad \text{and} \quad \behDist( \action_\numA ) = 2^{ -\numA + 1 }.
\end{align*}
In addition, the reward distribution is chosen to be
\begin{align*}
  \rewardDist( \action_1 ) = 0.99, \quad \text{and} \quad 
  \rewardDist(\action_i ) = \ber(1/2) \text{ for all } k \ge 2 . 
\end{align*}
Fix any sample count $8 \le \n \le 2^\numA$. There must exist some 
$1 \leq k \leq \numA$ such that $\behDist \inparen{ \action_{k}} \in 
[\frac{1}{\n}, \frac{2}{\n}]$. Consider the event $\calE_k \coloneqq 
\inbraces{\n(k) = 1}$. Then we have
\begin{align*}
  \P[\calE_k] = \n \cdot \inparen{1 - \behDist(\action_k)}^{\n-1} 
  \behDist(\action_k) \stackrel{(i)}{\ge} \inparen{1- \frac{2}{\n}}^{\n-1} 
  \stackrel{(ii)}{\ge} 0.1.
\end{align*} 
Here step (i) follows from the fact that $\behDist(\action_k) \in [\frac{1}{\n}, 
\frac{2}{\n}]$, and step (ii) is a result of the assumption $\n \ge 8$.
Thus we can lower bound the expected suboptimality of the plug-in rule as
\begin{align*}
  \E_\dataset \insquare{ \val(\optPolicy) - \val(\estPolicy_\mathsf{plug})} \ge 
  0.49 \cdot \P \insquare{\estPolicy_\mathsf{plug} \ne 1} \ge 0.49 \cdot 
  \P\insquare{ \calE_k \cap \{ \empReward(k) = 1 \} } \ge 0.02.
\end{align*}
This proves part (1).

\medskip 
\noindent \textbf{Proof of part (2). } We use \Cref{cor:tabular} to prove part (2). In the multi-armed bandit instance we construct, we have $S=1$ and $\mu(\pi^\star) = 1/2$. We set $\confLevel = 1/n$ to get the in-expectation guarantee:
\begin{align*}
  \E_\dataset \insquare{ \val(\optPolicy) - \val(\estPolicy_\infty) } \lesssim ( 1 - \confLevel ) \cdot \sqrt{\frac{
  \log (\numA /\confLevel) }{ \n }} + \confLevel \cdot 0.49 \lesssim \sqrt{ \frac{  \log (\numA \n) }{ \n } }.
\end{align*}
It remains to check the sample complexity requirement. Under the choice of $\confLevel = 1/n$, the sample complexity holds whenever $n \gtrsim \log n$, which is true if $n$ is sufficiently large (say, $n \ge 200$).

\end{proof}

\noindent The proof of \Cref{prop:separation-erm-gp-fixed} illustrates that pessimistic learning rules can outperform plug-in estimation when the optimal (or near-optimal) actions are well-represented in the dataset. In this case, the optimal action appears $1/2$ the time in expectation. \Cref{prop:separation-erm-gp-fixed} \emph{does not show} whether pessimism is better in a minimax sense. Indeed, the recent paper~\cite{xiao2021optimality} shows the intriguing result that pessimism, plug-in estimation, and even optimism (an algorithmic paradigm used in the \emph{online} setting) are equally ``minimax-optimal'', and there exist instances where any one of the three outperforms the other two.

\section{Counterexample to Theorem 4.3 in the paper~\cite{yin2021towards}}\label{sec:counterexample}
In this section, we present a counterexample to showcase that the statement 
in Theorem 4.3 in the paper~\cite{yin2021towards} cannot hold in general. 

We consider the simple two-armed bandit problem with Bernoulli reward distributions. 
Fix the instance $\CBinstance_1$ to be $\behDist(\action_1) = \behDist(\action_1) 
= 1/2$, $\rewardDist(\action_1) = \ber(1)$, and $\rewardDist(\action_2) = \ber(0)$. 
We make the following two observations. First, the optimal action in this instance 
$\CBinstance_1$ is $\action_1$. Second, the quantity $\complexity_1$ is given by
\begin{align*}
\frac{1}{ \sqrt{ \behDist(\action_1) } } = \sqrt{2}.
\end{align*}

To demonstrate that Theorem 4.3 in the paper~\cite{yin2021towards} cannot hold 
in general, it suffices to prove the following proposition.

\begin{proposition}
For any alternative instance $\CBinstance_2$, there exists a learning rule $\estPolicy$ 
such that 
\begin{align}\label{eq:yu-xiang-counter-claim}
	\max_{ \CBinstance \in \CBinstance_1, \CBinstance_2} 
	\E_{\dataset} \insquare { \optVal_{ \CBinstance} - \val_{ \CBinstance} ( 
	\estPolicy) } \leq e^{- \const \n} 
\end{align}
holds for some constant $\const > 0$.
\end{proposition}

\newcommand{\probOne}{\ensuremath{p}}
\newcommand{\BerOne}{\ensuremath{\alpha}}
\newcommand{\BerTwo}{\ensuremath{\beta}}

\begin{proof}
In general, we can parametrize the alternative instance $\CBinstance_2$ 
using the following three quantities
\begin{align*}
	\behDist_{\CBinstance_2}( \action_1 ) &= \probOne; \\
	\rewardDist_{\CBinstance_2}( \action_1 ) &= \ber( \BerOne ); \\
	\rewardDist_{\CBinstance_2}( \action_2 ) &= \ber( \BerTwo ),
\end{align*}
where $\probOne, \BerOne, \BerTwo$ are all between 0 and 1.

We make a key observation that: if $\BerOne \geq \BerTwo$ (i.e., the alternative 
instance $\CBinstance_2$ has the same optimal action $\action_1$ as for 
$\CBinstance_1$), then we could take 
$\estPolicy = \action_1$. The desired claim~\eqref{eq:yu-xiang-counter-claim} 
follows trivially since such $\estPolicy$ will incur zero suboptimality in both instances. 

Consequently, we only need to focus on the case when $\BerOne < \BerTwo$. 
For this case, we will explicitly construct an estimated policy $\estPolicy$ such 
that the claim~\eqref{eq:yu-xiang-counter-claim} is true. 
Let $\estPolicy_{\mathsf{LR}}$ be the optimal likelihood ratio test---based on the 
data $\dataset$---between the two instances $\CBinstance_1$, and 
$\CBinstance_2$. We claim that 
\begin{align*}
\max_{ \CBinstance \in \CBinstance_1, \CBinstance_2} 
	\E_{\dataset} \insquare { \optVal_{ \CBinstance} - \val_{ \CBinstance} ( 
	\estPolicy_{\mathsf{LR}}) } \leq e^{- \const \n}.
\end{align*} 
To see this, we observe that 
\begin{align*}
\max_{ \CBinstance \in \CBinstance_1, \CBinstance_2} 
	\E_{\dataset} \insquare { \optVal_{ \CBinstance} - \val_{ \CBinstance} ( 
	\estPolicy_{\mathsf{LR}}) } & \leq 
	\E_{\dataset} \insquare { \optVal_{ \CBinstance_1} - \val_{ \CBinstance_1} ( 
	\estPolicy_{\mathsf{LR}}) } + 
	\E_{\dataset} \insquare { \optVal_{ \CBinstance_2} - \val_{ \CBinstance_2} ( 
	\estPolicy_{\mathsf{LR}}) } \\
	& = \mathbb{P}_{\CBinstance_1} ( \estPolicy_{\mathsf{LR}} = \action_2 ) + 
	( \BerTwo - \BerOne ) \cdot \mathbb{P}_{\CBinstance_2} 
	( \estPolicy_{\mathsf{LR}} = \action_1 ) \\
	& \leq \mathbb{P}_{\CBinstance_1} ( \estPolicy_{\mathsf{LR}} = \action_2 ) + 
	\mathbb{P}_{\CBinstance_2} 
	( \estPolicy_{\mathsf{LR}} = \action_1 ).
\end{align*}
Here, the equality follows from the definition, and the last inequality 
uses the fact that $\BerTwo - \BerOne \leq 1$. 

By definition of $\estPolicy_{\mathsf{LR}}$, we have 
\begin{align*}
\mathbb{P}_{\CBinstance_1} ( \estPolicy_{\mathsf{LR}} = \action_2 ) + 
	\mathbb{P}_{\CBinstance_2} 
	( \estPolicy_{\mathsf{LR}} = \action_1 ) = 1 - \mathsf{TV} ( 
	\CBinstance_1^{\n} \| \CBinstance_2^{\n}). 
\end{align*}
Therefore it suffices to prove that 
\begin{align}\label{eq:TV-bound}
1 - \mathsf{TV} ( 
	\CBinstance_1^{\n} \| \CBinstance_2^{\n}) \leq e^{-\const \n}.
\end{align}

\paragraph{Proof of the bound~\eqref{eq:TV-bound}. } 

By the relation between the total variation distance and the Hellinger 
distance, we have
\begin{align*}
1 - \mathsf{TV} ( \CBinstance_1^{\n} \| \CBinstance_2^{\n} ) \leq 
1 - \mathsf{Hel}^2 ( \CBinstance_1^{\n} \| \CBinstance_2^{\n} ) = 
\insquare{ 1 - \mathsf{Hel}^2 ( \CBinstance_1 \| \CBinstance_2 ) }^{\n},
\end{align*}
where the inequality arises from the property of the Hellinger distance. 
By constructions of $\CBinstance_1, \CBinstance_2$, and the definition 
of the Hellinger distance, one has 
\begin{align*}
\mathsf{Hel}^2 ( \CBinstance_1 \| \CBinstance_2 ) = \frac{1}{2} 
\left \{ \inparen{ \frac{1}{ \sqrt{2} } - \sqrt{ \probOne \BerOne} }^2 
+  \probOne (1 - \BerOne) + (1 - \probOne) \BerTwo 
+ \inparen{ \frac{1}{ \sqrt{2} } - \sqrt{ (1 - \probOne) (1 - \BerTwo)} }^2 \right \}.
\end{align*}
Note that 
\begin{align*}
\inf_{\probOne, \BerOne, \BerTwo: \BerTwo > \BerOne} 
\mathsf{Hel}^2 ( \CBinstance_1 \| \CBinstance_2 ) = 1 - \frac{1}{ \sqrt{2} }. 
\end{align*}
We can then combine the previous relations to finish the proof 
of the bound~\eqref{eq:TV-bound}. 
\end{proof}

\end{document}